%% file: main.tex
\documentclass{article}

\usepackage[final]{neurips_2020}

\usepackage{todonotes}

\usepackage{natbib}

\usepackage[utf8]{inputenc} 
\usepackage[T1]{fontenc}    
\usepackage{hyperref}       
\usepackage{url}            
\usepackage{booktabs}       
\usepackage{amsfonts}       
\usepackage{nicefrac}       
\usepackage{microtype}      

\usepackage{graphicx}
\usepackage{caption}
\usepackage{subcaption}

\usepackage[ruled,vlined]{algorithm2e}

\include{notation}

\title{Emergent Complexity and Zero-shot Transfer via Unsupervised Environment Design}

\author{%
  Michael Dennis\thanks{Equal contribution},~ Natasha Jaques$^{*2}$, Eugene Vinitsky, Alexandre Bayen, \\
  \textbf{Stuart Russell, Andrew Critch, Sergey Levine} \\
  University of California Berkeley AI Research (BAIR), Berkeley, CA, 94704\\
  $^2$Google Research, Brain team, Mountain View, CA, 94043 \\
  \texttt{michael\_dennis@berkeley.edu, natashajaques@google.com,} \\
  \texttt{$<$evinitsky, bayen, russell, critch, svlevine$>$@berkeley.edu} \\
}

\begin{document}
\setcitestyle{numbers}

\maketitle

\begin{abstract}
A wide range of reinforcement learning (RL) problems --- including robustness, transfer learning, unsupervised RL, and emergent complexity --- require specifying a distribution of tasks or environments in which a policy will be trained.  However, creating a useful distribution of environments is error prone, and takes a significant amount of developer time and effort. We propose Unsupervised Environment Design (UED) as an alternative paradigm, where developers provide environments with unknown parameters, and these parameters are used to automatically produce a distribution over valid, solvable environments. Existing approaches to automatically generating environments suffer from common failure modes: domain randomization cannot generate structure or adapt the difficulty of the environment to the agent's learning progress, and minimax adversarial training leads to worst-case environments that are often unsolvable.
To generate structured, solvable environments for our \textit{protagonist} agent, we introduce a second, \textit{antagonist} agent that is allied with the \textit{environment-generating adversary}. The adversary is motivated to generate environments which maximize regret, defined as the difference between the protagonist and antagonist agent's return. We call our technique Protagonist Antagonist Induced Regret Environment Design (PAIRED).
Our experiments demonstrate that PAIRED produces a natural curriculum of increasingly complex environments, and PAIRED agents achieve higher zero-shot transfer performance when tested in highly novel environments.
\end{abstract}

\section{Introduction}

Many reinforcement learning problems require designing a distribution of tasks and environments that can be used to evaluate and train effective policies. This is true for a diverse array of methods including transfer learning (e.g., \citep{antonova2017reinforcement,yu2017preparing,sadeghi2016cad2rl,tobin2017domain}), robust RL (e.g., \citep{bagnell2001solving, iyengar2005robust,pinto2017supervision}), unsupervised RL (e.g., \citep{gupta2018unsupervised}), and emergent complexity (e.g., \citep{sukhbaatar2017intrinsic,wang2019paired,wang2020enhanced}). For example, suppose we wish to train a robot in simulation to pick up objects from a bin in a real-world warehouse. There are many possible configurations of objects, including objects being stacked on top of each other. We may not know \textit{a priori} the typical arrangement of the objects, but can naturally describe the simulated environment as having a distribution over the object positions.

However, designing an appropriate distribution of environments is challenging.  The real world is complicated, and correctly enumerating all of the edge cases relevant to an application could be impractical or impossible. Even if the developer of the RL method knew every edge case, specifying this distribution could take a significant amount of time and effort.  We want to automate this process.

\captionsetup[sub]{font=small}
\begin{figure}
\centering
\begin{subfigure}{.24\textwidth}
  \centering
  \includegraphics[width=.9\linewidth]{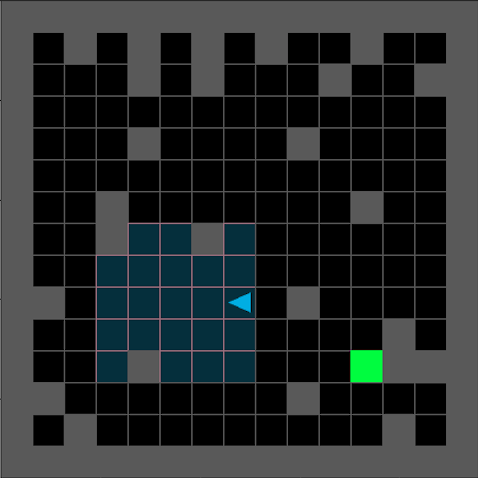}
  \caption{\footnotesize{Domain Randomization}}
  \label{fig:ex_domain_rand}
\end{subfigure}%
\begin{subfigure}{.239\textwidth}
  \centering
  \includegraphics[width=.9\linewidth]{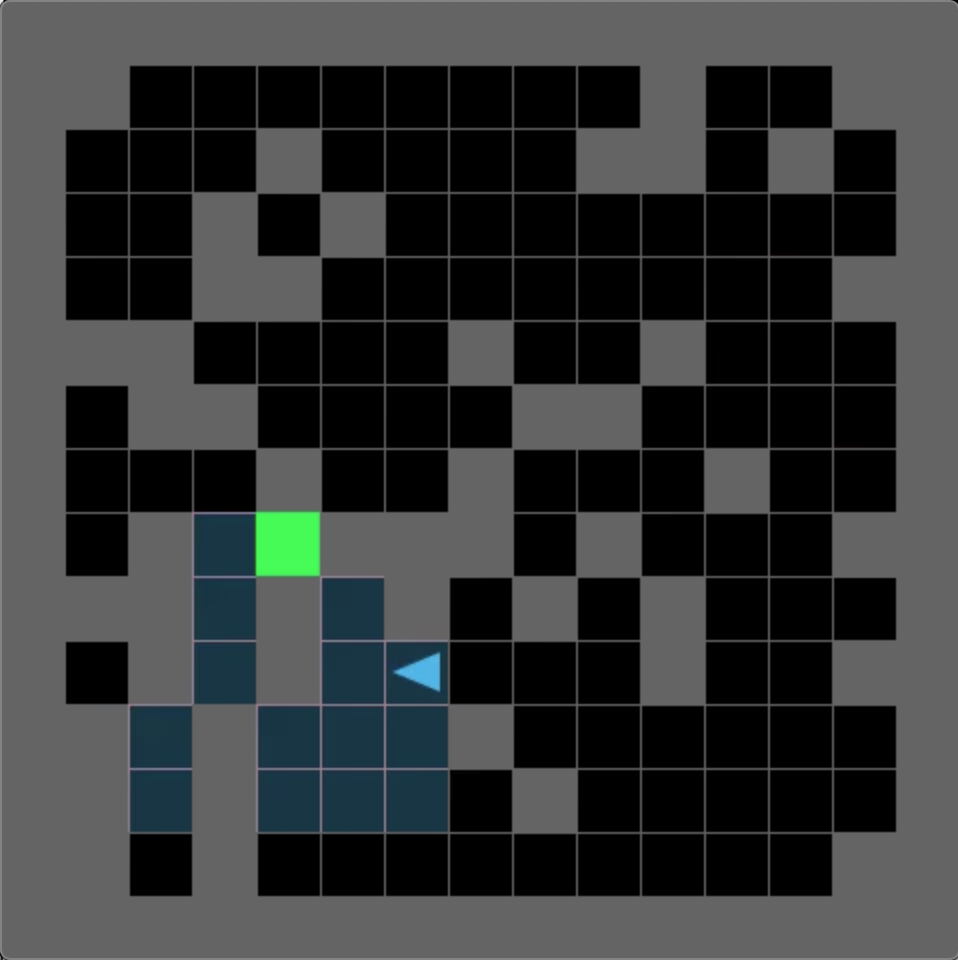}
  \caption{Minimax Adversarial}
  \label{fig:ex_unconstrained}
\end{subfigure}%
\begin{subfigure}{.24\textwidth}
  \centering
  \includegraphics[width=.9\linewidth]{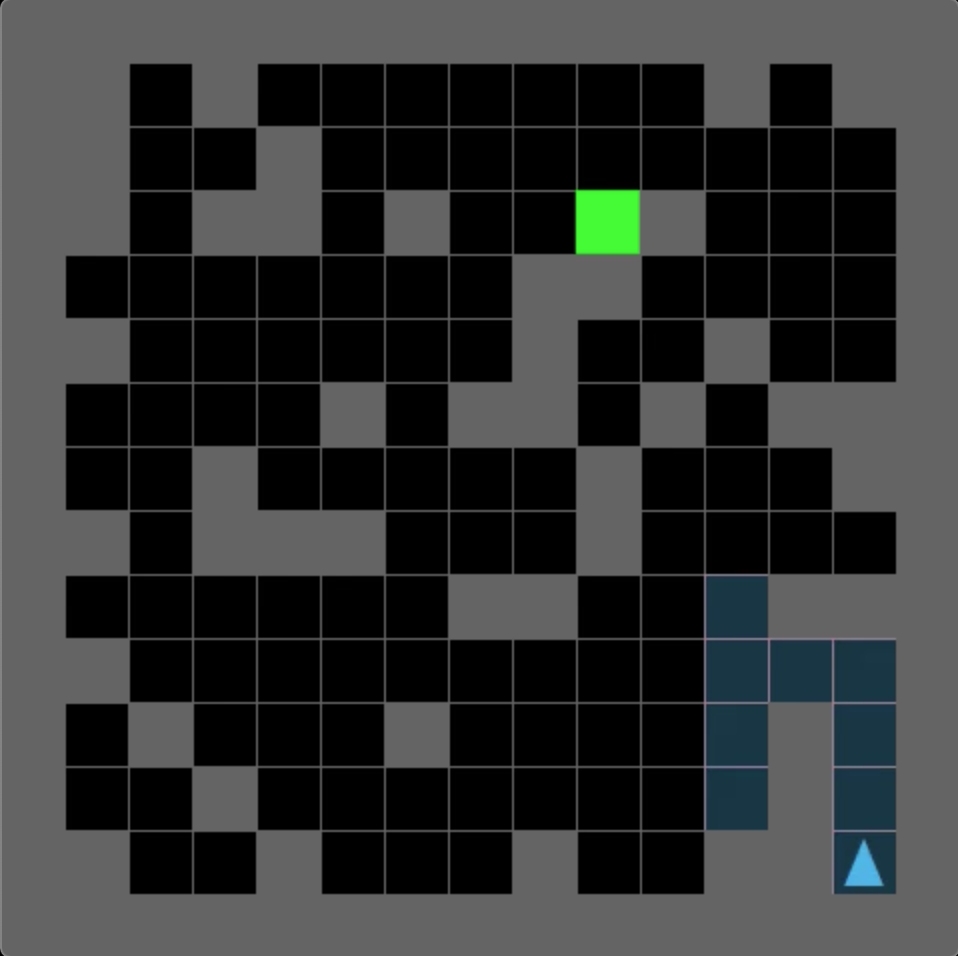}
  \caption{PAIRED (ours)}
  \label{fig:ex_PAIRED}
\end{subfigure}%
\begin{subfigure}{.24\textwidth}
  \centering
  \includegraphics[width=.95\linewidth]{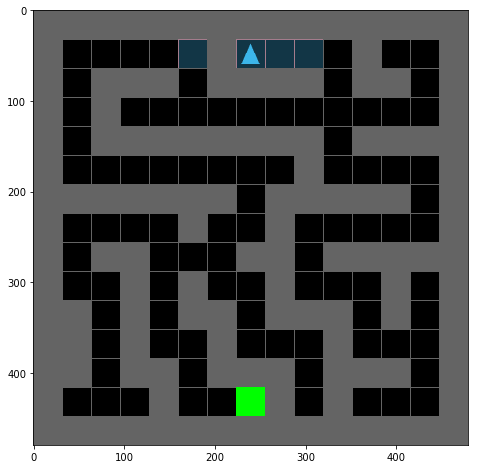}
  \caption{Transfer task}
  \label{fig:ex_maze}
\end{subfigure}
\caption{An agent learns to navigate an environment
where the position of the goal and obstacles is an underspecified parameter. If trained using domain randomization to randomly choose the obstacle locations (a), the agent will fail to generalize to a specific or complex configuration of obstacles, such as a maze (d). Minimax adversarial training encourages the adversary to create impossible environments, as shown in (b). In contrast, Protagonist Antagonist Induced Regret Environment Design (PAIRED),
trains the adversary based on the difference between the reward of the agent (protagonist) and the reward of a second, antagonist agent. Because the two agents are learning together, the adversary is motivated to create a curriculum of difficult but achievable environments tailored to the agents' current level of performance (c). PAIRED facilitates learning complex behaviors and policies that perform well under zero-shot transfer to challenging new environments at test time.}
\vspace{-0.3cm}
\label{fig:example_grids}
\end{figure}

In order for this automation to be useful, there must be a way of specifying the domain of environments in which the policy should be trained, without needing to fully specify the distribution.
In our approach, developers only need to supply an \emph{underspecified environment}: an environment which has free parameters which control its features and behavior. For instance, the developer could give a navigation environment in which the agent's objective is to navigate to the goal and the free parameters are the positions of obstacles.  Our method will then construct distributions of environments by providing a distribution of settings of these free parameters; in this case, positions for those blocks.  We call this problem of taking the underspecified environment and a policy,
and producing an interesting distribution of fully specified environments in which that policy can be further trained \emph{Unsupervised Environment Design} (UED).  We formalize unsupervised environment design in Section \ref{formalisms}, providing a characterization of the space of possible approaches which subsumes prior work. After training a policy in a distribution of environments generated by UED, we arrive at an updated policy and can use UED to generate more environments in which the updated policy can be trained.  In this way, an approach for UED naturally gives rise to an approach for Unsupervised Curriculum Design (UCD).  This method can be used to generate capable policies through increasingly complex environments targeted at the policy's current abilities.

Two prior approaches
to UED are domain randomization, which generates fully specified environments uniformly randomly regardless of the current policy (e.g., \citep{jakobi1997evolutionary,sadeghi2016cad2rl,tobin2017domain}), and adversarially generating environments to minimize the reward of the current policy; \textit{i.e.} minimax training (e.g., \citep{pinto2017supervision, morimoto2001robust,vinitsky2020dmalt,khirodkar2018vadra}).
While each of these approaches have their place, they can each fail to generate any interesting environments. In Figure \ref{fig:example_grids} we show examples of maze navigation environments generated by each of these techniques.  Uniformly random environments will often fail to generate interesting structures; in the maze example, it will be unlikely to generate walls (Figure \ref{fig:ex_domain_rand}).
On the other extreme, a minimax adversary is incentivized to make the environments completely unsolvable, generating mazes with unreachable goals (Figure \ref{fig:ex_unconstrained}).  In many environments, both of these methods fail to generate structured and solvable environments. We present a middle ground, generating environments which maximize regret, which produces difficult but solvable tasks (Figure \ref{fig:ex_PAIRED}). Our results show that optimizing regret results in agents that are able to perform difficult transfer task (Figure \ref{fig:ex_maze}), which are not possible using the other two techniques.  

We propose a novel adversarial training technique which naturally solves the problem of the adversary generating unsolvable environments by introducing an \emph{antagonist} which is allied with the \emph{environment-generating adversary}.  For the sake of clarity, we refer to the primary agent we are trying to train as the \emph{protagonist}.  The environment adversary's goal is to design environments in which the antagonist achieves high reward and the protagonist receives low reward.  If the adversary generates unsolvable environments, the antagonist and protagonist would perform the same and the adversary would get a score of zero, but if the adversary finds environments the antagonist solves and the protagonist does not solve, the adversary achieves a positive score. Thus, the environment adversary is incentivized to create challenging but \textit{feasible} environments, in which the antagonist can outperform the protagonist. Moreover, as the protagonist learns to solves the simple environments, the antagonist must generate more complex environments to make the protagonist fail, increasing the complexity of the generated tasks and leading to automatic curriculum generation.

We show that when both teams reach a Nash equilibrium, the generated environments have the highest regret, in which the performance of the protagonist's policy most differs from that of the optimal policy. 
Thus, we call our approach \emph{Protagonist Antagonist Induced Regret Environment Design (PAIRED)}. Our results demonstrate that compared to domain randomization, minimax adversarial training, and population-based minimax training (as in \citet{wang2019paired}), PAIRED agents learn more complex behaviors, and have higher zero-shot transfer performance in challenging, novel environments.

In summary, our main contributions are: a) formalizing the problem of Unsupervised Environment Design (UED) (Section~\ref{formalisms}), b) introducing the PAIRED Algorithm (Section~\ref{sec:paired}), c) demonstrating PAIRED leads to more complex behavior and more effective zero-shot transfer to new environments than existing approaches to environment design (Section~\ref{sec:results}), and d) characterizing solutions to UED and connecting the framework to classical decision theory (Section~\ref{formalisms}).

\section{Related Work}
When proposing an approach to UED we are essentially making a decision about which environments to prioritize during training, based on an underspecified environment set provided by the designer.  Making decisions in situations of extreme uncertainty has been studied in decision theory under the name ``decisions under ignorance''~\citep{peterson2017introduction} and is in fact deeply connected to our work, as we detail in Section \ref{formalisms}. \citet{milnor1954chapter} characterize several of these approaches; 
using Milnor's terminology, we see domain randomization \citep{jakobi1997evolutionary} as ``the principle of insufficient reason'' proposed by Laplace, minimax adversarial training \citep{khirodkar2018vadra} as the ``maximin principle'' proposed by Wald~\citep{wald1950statistical}, and our approach, PAIRED, as ``minimax regret'' proposed by Savage~\citep{savage1951theory}. Savage's approach was built into a theory of making sequential decisions which minimize worst case regret~\citep{robbins1952some, lai1985asymptotically}, which eventually developed into modern literature on minimizing worst case regret in bandit settings \citep{bubeck2012regret}.

Multi-agent training has been proposed as a way to automatically generate curricula in RL~\citep{leibo2019autocurricula,matiisen2019teacher,graves2017automated}. Competition can drive the emergence of complex skills, even some that the developer did not anticipate \citep{baker2019emergent,gleave2019adversarial}. Asymmetric Self Play (ASP) \citep{sukhbaatar2017intrinsic} trains a demonstrator agent to complete the simplest possible task that a learner cannot yet complete, ensuring that the task is, in principle, solvable. In contrast, PAIRED is significantly more general because it generates complex new environments, rather than trajectories within an existing, limited environment. \citet{campero2020learning} study curriculum learning in a similar environment to the one under investigation in this paper. They use a teacher to propose goal locations to which the agent must navigate, and the teacher's reward is computed based on whether the agent takes more or less steps than a threshold value. To create a curriculum, the threshold is linearly increased over the course of training. POET \citep{wang2019paired,wang2020enhanced} 
uses a population of minimax (rather than minimax \emph{regret}) adversaries to generate the terrain for a 2D walker. 
However, POET \cite{wang2019paired} requires generating many new environments, testing all agents within each one, and discarding environments based on a manually chosen reward threshold, which wastes a significant amount of computation. 
In contrast, our minimax regret approach automatically learns to tune the difficulty of the environment by comparing the protagonist and antagonist's scores. In addition, these works do not investigate whether adversarial environment generation can provide enhanced generalization when agents are transferred to new environments, as we do in this paper. 

Environment design has also been investigated within the symbolic AI community. The problem of modifying an environment to influence an agent's decisions is analyzed in \citet{zhang2009general}, and was later extended through the concept of Value-based Policy Teaching \cite{zhang2008value}.
\citet{keren2017equi,keren2019efficient} use a heuristic best-first search to redesign an MDP to maximize expected value.

We evaluated PAIRED in terms of its ability to learn policies that transfer to unseen, hand-designed test environments.
One approach to produce policies which can transfer between environments is unsupervised RL (e.g., \citep{schmidhuber2010formal,sermanet2018time}). 
\citet{gupta2018unsupervised} propose an unsupervised meta-RL technique that computes minimax regret against a diverse task distribution learned with the DIAYN algorithm \citep{eysenbach2018diversity}. However, this algorithm does not learn to modify the environment and does not adapt the task distribution based on the learning progress of the agent. PAIRED provides a curriculum of increasingly difficult environments and allows us to provide a theoretical characterization of when minimax regret should be preferred over other approaches.

The robust control literature has made use of both minimax and regret objectives.
Minimax training has been used in the controls literature~\citep{slotine1987adaptive, aastrom1983theory,zhou1996robust}, the Robust MDP literature~\citep{bagnell2001solving, iyengar2005robust,nilim2005robust,tamar2014scaling}, and more recently, through the use of adversarial RL training~\cite{pinto2017supervision,vinitsky2020dmalt,morimoto2001robust}. 
Unlike our algorithm, minimax adversaries have no incentive to guide the learning progress of the agent and can make environments arbitrarily hard or unsolvable. 
\citet{ghavamzadeh2016safe} minimize the regret of a model-based policy against a safe baseline to ensure safe policy improvement. \citet{regan2011robust,regan2012regret} study Markov Decision Processes (MDPs) in which the reward function is not fully specified, and use minimax regret as an objective to guide the elicitation of user preferences. While these approaches use similar theoretical techniques to ours, they use analytical solution methods that do not scale to the type of deep learning approach used in this paper, do not attempt to learn to generate environments, and do not consider automatic curriculum generation.

Domain randomization (DR) \citep{jakobi1997evolutionary} is an alternative approach in which a designer specifies a set of parameters to which the policy should be robust \citep{sadeghi2016cad2rl,tobin2017domain,tan2018sim}. These parameters are then randomly sampled for each episode, and a policy is trained that performs well on average across parameter values. However, this does not guarantee good performance on a specific, challenging configuration of parameters. While DR has had empirical success \citep{andrychowicz2020learning}, it requires careful parameterization and does not automatically tailor generated environments to the current proficiency of the learning agent. \citet{mehta2020active} propose to enhance DR by learning which parameter values lead to the biggest decrease in performance compared to a reference environment.

\section{Unsupervised Environment Design}
\label{formalisms}

The goal of this work is to construct a policy that performs well across a large set of environments. We train policies by starting with an initially random policy, generating environments based on that policy to best suit its continued learning, train that policy in the generated environments, and repeat until the policy converges or we run out of resources. We then test the trained policies in a set of challenging transfer tasks not provided during training.  In this section, we will focus on the environment generation step, which we call \emph{unsupervised environment design} (UED).  We will formally define UED as the problem of using an underspecified environment
to produce a distribution over fully specified environments, which supports the continued learning of a particular policy.  
To do this, we must formally define fully specified and underspecified environments, and describe how to create a distribution of environments using UED. We will end this section by proposing minimax regret as a novel approach to UED.

We will model our fully specified environments with a Partially Observable Markov Decision Process (POMDP), which is a tuple $\langle \As,\Os,\Ss{}, \Tf{},\Of{},\Rf{},\discount \rangle$ where $\As$ is a set of actions, $\Os$ is a set of observations, $\Ss{}$ is a set of states, $\Tf{}: S \times A \rightarrow \Dist{S}$ is a transition function, $\Of{}: S \rightarrow O$ is an observation (or inspection) function, $\Rf{}: S \rightarrow \mathbb{R}$, and $\discount$ is a discount factor.  We will define the utility as $\Uf{}(\pi) = \sum_{i=0}^{T} r_t\gamma^t$, where $T$ is a horizon.

To model an underspecified environment, we propose the Underspecified Partially Observable Markov Decision Process (UPOMDP) as a tuple $\PPOMDP = \langle \As,\Os, \Ns, \Ss{\PPOMDP}, \Tf{\PPOMDP},\Of{\PPOMDP},\Rf{\PPOMDP},\discount \rangle$. 
The only difference between a POMDP and a UPOMDP is that a UPOMDP has a set $\Ns$ representing the free parameters of the environment, which can be chosen to be distinct at each time step and are incorporated into the transition function as $\Tf{\PPOMDP}: S \times A \times \Ns \rightarrow \Dist{S}$.  Thus a possible setting of the environment is given by some trajectory of environment parameters $\vec{\theta}$.  As an example UPOMDP, consider a simulation of a robot in which $\vec{\theta}$ are additional forces which can be applied at each time step.  A setting of the environment $\vec{\theta}$ can be naturally combined with the underspecified environment $\PPOMDP$ to give a specific POMDP, which we will denote $\apply{\PPOMDP}{\vec{\theta}}$.

In UED, we want to generate a distribution of environments in which to continue training a policy.  We would like to curate this distribution of environments to best support the continued learning of a particular agent policy.  As such, we 
can propose a solution to UED by 
specifying some \emph{environment policy}, $\Lambda: \Pi \rightarrow \Dist{\Theta^T}$ where $\Pi$ is the set of possible policies
and $\Theta^T$ is the set of possible sequences of environment parameters. 
In Table \ref{Lambda_examples}, we see a few examples of possible choices for environment policies, as well as how they correspond to previous literature.  
Each of these choices also have a corresponding decision rule used for making \emph{decisions under ignorance}~\citep{peterson2017introduction}.  Each decision rule can be understood as a method for choosing a policy given an underspecified environment.  This connection between UED and decisions under ignorance extends further than these few decision rules.  In the Appendix we will make this connection concrete and show that, under reasonable assumptions, UED and decisions under ignorance are solving the same problem.

\begin{table}
  \caption{Example Environment Policies}
  
  \centering
  \begin{tabular}{lll}
    \toprule 
    UED Technique & Environment Policy    & Decision Rule    \\
    \midrule
    Domain Randomization \cite{jakobi1997evolutionary, sadeghi2016cad2rl,tobin2017domain} &  $\Lambda^{DR}(\pi) = \mathcal{U}(\Theta^T)$ & Insufficient Reason    \\
    Minimax Adversary \cite{pinto2017supervision,vinitsky2020dmalt,morimoto2001robust}& $\Lambda^{M}(\pi) = \argmin\limits_{\vec{\theta} \in \Theta^T}\{\Uf{\vec{\theta}}(\pi)\}$     & Maximin \\
    PAIRED (ours) &
    $\Lambda^{MR}(\pi) = \{\overline{\theta}_{\pi}: \frac{c_\pi}{v_\pi}, \tilde{D}_{\pi}:\text{otherwise} \} $ 
    & Minimax Regret    \\
    \bottomrule
  \end{tabular}
  \caption{The environment policies corresponding to the three techniques for UED which we study, along with their corresponding decision rule. Where $\mathcal{U}(X)$ is a uniform distribution over X, $\tilde{D}_{\pi}$ is a baseline distribution, $\overline{\theta}_{\pi}$ is the trajectory which maximizes regret of $\pi$, $v_\pi$ is the value above the baseline distribution that $\pi$ achieves on that trajectory, and $c_\pi$ is the negative of the worst-case regret of $\pi$, normalized so that $\frac{c_\pi}{v_\pi}$ is between $0$ and $1$.  This is described in full in the appendix.}
  \label{Lambda_examples}
\end{table}

PAIRED can be seen as an approach for approximating the environment policy, $\Lambda^{MR}$, which corresponds to the decision rule minimax regret.  One useful property of minimax regret, which is not true of domain randomization or minimax adversarial training, is that whenever a task has a sufficiently well-defined notion of success and failure it chooses policies which succeed.  By minimizing the worst case difference between what is achievable and what it achieves, whenever there is a policy which ensures success, it will not fail where others could succeed.

\begin{theorem}
\label{success:theorem}
 Suppose that all achievable rewards fall into one of two class of outcomes labeled \textbf{SUCCESS} giving rewards in $[\textbf{S}_{min}, \textbf{S}_{max}]$ and \textbf{FAILURE} giving rewards in $[\textbf{F}_{min}, \textbf{F}_{max}]$, such that $\textbf{F}_{min} \leq \textbf{F}_{max} < \textbf{S}_{min}\leq \textbf{S}_{max}$.  In addition assume that the range of possible rewards in either class is smaller than the difference between the classes so we have $\textbf{S}_{max} - \textbf{S}_{min} < \textbf{S}_{min} - \textbf{F}_{max}$ and $ \textbf{F}_{max} - \textbf{F}_{min} < \textbf{S}_{min} - \textbf{F}_{max}$. Further suppose that there is a policy $\pi$ which succeeds on any $\vec{\theta}$ whenever success is possible. Then minimax regret will choose a policy which has that property.
\end{theorem}

The proof of this property, and examples showing how minimax and domain randomization fail to have this property, are in the Appendix. In the next section we will formally introduce PAIRED as a method for approximating the minimax regret environment policy, $\Lambda^{MR}$.

\section{Protagonist Antagonist Induced Regret Environment Design (PAIRED)}
\label{sec:paired} 
Here, we describe how to approximate minimax regret, and introduce our proposed algorithm, Protagonist Antagonist Induced Regret Environment Design (PAIRED).  Regret is defined as the difference between the payoff obtained for a decision, and the optimal payoff that could have been obtained in the same setting with a different decision. In order to approximate regret, we use the difference between the payoffs of two agents acting under the same environment conditions.  Assume we are given a fixed environment with parameters $\vec{\theta}$, a fixed policy for the protagonist agent, $\pi^P$, and we then train a second antagonist agent, $\pi^A$, to optimality in this environment. Then,
the difference between the reward obtained by the antagonist, $U^{\vec{\theta}}(\pi^A)$, and the protagonist, $U^{\vec{\theta}}(\pi^P)$, is the regret:
\begin{align}
    \normalfont\textsc{Regret}^{\vec{\theta}}\left(\pi^P, \pi^A\right) = U^{\vec{\theta}}\left(\pi^A\right) - U^{\vec{\theta}}\left(\pi^P\right)
    \label{eq:paired_regret}
\end{align}
In PAIRED, we introduce an environment adversary that learns to control the parameters of the environment, $\vec{\theta}$, to maximize the regret of the protagonist against the antagonist. 
For each training batch,
the adversary generates the parameters of the environment, $\vec{\theta} \sim \tilde{\Lambda}$, which both agents will play. The adversary and antagonist are trained to maximize the regret as computed in Eq. \ref{eq:paired_regret}, while the protagonist learns to minimize regret. This procedure is shown in Algorithm \ref{alg:paired}. The code for PAIRED and our experiments is available in open source at \url{https://github.com/google-research/google-research/tree/master/social_rl/}.
 We note that our approach is agnostic to the choice of RL technique used to optimize regret. 

To improve the learning signal for the adversary, once the adversary creates an environment, both the protagonist and antagonist generate several trajectories within that same environment. This allows for a more accurate approximation the minimax regret as the difference between the maximum reward of the antagonist and the average reward of the protagonist over all trajectories: $\textsc{Regret} \approx \max_{\tau^A} U^{\vec{\theta}}(\tau^A) - \mathbb{E}_{\tau^P}[U^{\vec{\theta}}(\tau^P)]$. We have found this reduces noise in the reward and more accurately rewards the adversary for building difficult but solvable environments.

\begin{algorithm}[H]
\SetAlgoLined
 Randomly initialize Protagonist $\pi^P$, Antagonist $\pi^A$, and Adversary $\tilde{\Lambda}$; \\
 \While{not converged}{
  Use adversary to generate environment parameters: $\vec{\theta} \sim \tilde{\Lambda}$. Use to create POMDP $\apply{\PPOMDP}{\vec{\theta}}$.\\
  Collect Protagonist trajectory $\tau^P$ in $\apply{\PPOMDP}{\vec{\theta}}$. Compute: $U^{\vec{\theta}}(\pi^P) = \sum_{i=0}^{T} r_t\gamma^t$ \\
  Collect Antagonist trajectory $\tau^A$ in $\apply{\PPOMDP}{\vec{\theta}}$. Compute: $U^{\vec{\theta}}(\pi^A) = \sum_{i=0}^{T} r_t\gamma^t$ \\
  Compute: $\normalfont\textsc{Regret}^{\vec{\theta}}(\pi^P, \pi^A) = U^{\vec{\theta}}(\pi^A) -U^{\vec{\theta}}(\pi^P) $\\
  Train Protagonist policy $\pi^P$ with RL update and reward $R(\tau^P)= -\normalfont\textsc{Regret}$ \\
  Train Antagonist policy $\pi^A$ with RL update and reward $R(\tau^A)= \normalfont\textsc{Regret}$ \\
  Train Adversary policy $\tilde{\Lambda}$ with RL update and reward $R(\tau^{\tilde{\Lambda}})= \normalfont\textsc{Regret}$
 }
 \caption{PAIRED.}
 \label{alg:paired}
\end{algorithm}

At each training step, the environment adversary can be seen as solving a UED problem for the current protagonist, generating a curated set of environments in which it can learn. While the adversary is motivated to generate tasks beyond the protagonist's abilities, the regret can actually incentivize creating the easiest task on which the protagonist fails but the antagonist succeeds. This is because if the reward function contains any bonus for solving the task more efficiently (e.g. in fewer timesteps), the antagonist will get more reward if the task is easier. Thus, the adversary gets the most regret for proposing the easiest task which is outside the protagonist's skill range. Thus, regret incentivizes the adversary to propose tasks within the agent's ``zone of proximal development'' \cite{chaiklin2003zone}. As the protagonist learns to solve the simple tasks, the adversary is forced to find harder tasks to achieve positive reward, increasing the complexity of the generated tasks and leading to automatic curriculum generation.  

Though multi-agent learning may not always converge \citep{mazumdar2020gradient}, we can show that, if each team in this game finds an optimal solution, the protagonist would be playing a minimax regret policy.
\begin{theorem}
    Let $(\pi^P,\pi^A,\vec{\theta})$ be in Nash equilibrium and the pair $(\pi^A,\vec{\theta})$ be jointly a best response to $\pi^P$.  Then $\pi^P \in \argmin\limits_{\pi^P \in \Pi^P}\{\argmax\limits_{\pi^A,\vec{\theta} \in \Pi^A \times \Theta^T}\{\normalfont\textsc{Regret}^{\vec{\theta}}(\pi^P, \pi^A)\}\}$.
\end{theorem}
\begin{proof}
    Let $(\pi^P,\pi^A,\vec{\theta})$ be in Nash equilibria and the pair $(\pi^A,\vec{\theta})$ is jointly a best response to $\pi^P$.  Then we can consider $(\pi^A,\vec{\theta})$ as one player with policy which we will write as $\pi^{A+\vec{\theta}} \in \Pi \times \Theta^T$.  Then $\pi^{A+\vec{\theta}}$ is in a zero sum game with $\pi^P$, and the condition that the pair $(\pi^A,\vec{\theta})$ is jointly a best response to $\pi^P$ is equivalent to saying that $\pi^{A+\vec{\theta}}$ is a best response to $\pi^P$.  Thus $\pi^P$ and $\pi^{A+\vec{\theta}}$ form a Nash equilibria of a zero sum game, and the minimax theorem applies.  Since the reward in this game is defined by $\normalfont\textsc{Regret}$ we have:
    \[\pi^P \in \argmin\limits_{\pi^P \in \Pi^P}\{\argmax\limits_{\pi^{A+\vec{\theta}} \in \Pi \times \Theta^T}\{\normalfont\textsc{Regret}^{\vec{\theta}}(\pi^P, \pi^A)\}\}\]
    By the definition of $\pi^{A+\vec{\theta}}$ this proves the protagonist learns the minimax regret policy.
\end{proof}

The proof gives us reason to believe that the iterative PAIRED
training process in Algorithm \ref{alg:paired} can produce minimax regret policies. In Appendix \ref{Nash_PAIRED} we show that if the agents are in Nash equilibrium, without the coordination assumption, the protagonist will perform at least as well as the antagonist in every parameterization, and Appendix \ref{sec:app_combined} provides empirical results for alternative methods for approximating regret which break the coordination assumption.  In the following sections, we will show that empirically, policies trained with Algorithm \ref{alg:paired} exhibit good performance and transfer, both in generating emergent complexity and in training robust policies.

\section{Experiments}
\label{sec:results}

The experiments in this section focus on assessing whether training with PAIRED can increase the complexity of agents' learned behavior, and whether PAIRED agents can achieve better or more robust performance when transferred to novel environments. To study these questions, we first focus on the navigation tasks shown in Figure \ref{fig:example_grids}. Section \ref{sec:cont_control} presents results in continuous domains.

\subsection{Partially Observable Navigation Tasks}
Here we investigate navigation tasks (based on \citep{gym_minigrid}), in which an agent must explore to find a goal (green square in Figure \ref{fig:example_grids}) while navigating around obstacles. The environments are partially observable; the agent's field of view is shown as a blue shaded area in the figure. To deal with the partial observability, we parameterize the protagonist and antagonist's policies using recurrent neural networks (RNNs). All agents are trained with PPO \citep{schulman2017proximal}. Further details about network architecture and hyperparameters are given in Appendix \ref{sec:app_hparams}.

We train adversaries that learn to build these environments by choosing the location of the obstacles, the goal, and the starting location of the agent. The adversary's observations consist of a fully observed view of the environment state, the current timestep $t$, and a random vector $z \sim \mathcal{N}(0,I), z \in \mathbb{R}^D$ sampled for each episode. At each timestep, the adversary outputs the location where the next object will be placed; at timestep 0 it places the agent, 1 the goal, and every step afterwards an obstacle. Videos of environments being constructed and transfer performance are available at \url{https://www.youtube.com/channel/UCI6dkF8eNrCz6XiBJlV9fmw/videos}.

\subsubsection{Comparison to Prior Methods}
To compare to prior work that uses pure minimax training~\citep{pinto2017supervision,vinitsky2020dmalt,morimoto2001robust} (rather than minimax \emph{regret}), we use the same parameterization of the environment adversary and protagonist, but simply remove the antagonist agent. The adversary's reward is $R(\Lambda) = - \mathbb{E}_{\tau^P}[U(\tau^P)]$. While a direct comparison to POET~\cite{wang2019paired} is challenging since many elements of the POET algorithm are specific to a 2D walker, the main algorithmic distinction between POET and minimax environment design is maintaining a population of environment adversaries and agents that are periodically swapped with each other. Therefore, we also employ a Population Based Training (PBT) minimax technique in which the agent and adversary are sampled from respective populations for each episode. This baseline is our closest approximation of POET~\cite{wang2019paired}. To apply domain randomization, we simply sample the $(x,y)$ positions of the  agent, goal, and blocks uniformly at random. We sweep shared hyperparameters for all methods equally. Parameters for the emergent complexity task are selected to maximize the solved path length, and parameters for the transfer task are selected using a set of validation environments. Details are given in the appendix.

\captionsetup[sub]{font=small}
\begin{figure}
\centering
\begin{subfigure}{.25\textwidth}
  \centering
  \includegraphics[width=\linewidth]{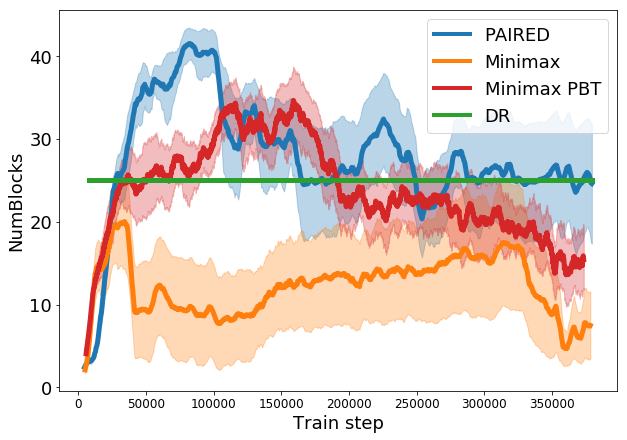}
  \caption{Number of blocks}
\end{subfigure}%
\begin{subfigure}{.25\textwidth}
  \centering
  \includegraphics[width=\linewidth]{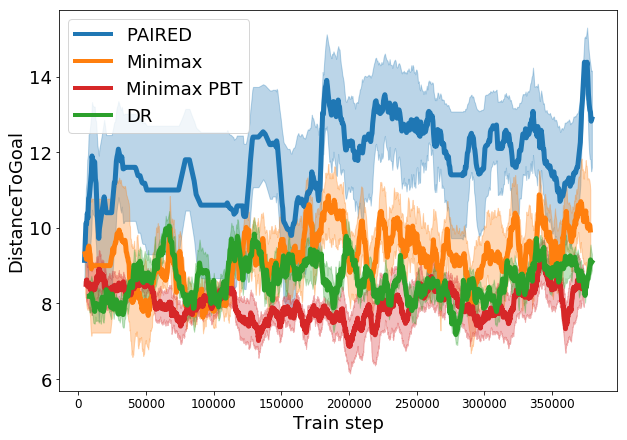}
  \caption{Distance to goal}
\end{subfigure}%
\begin{subfigure}{.25\textwidth}
  \centering
  \includegraphics[width=\linewidth]{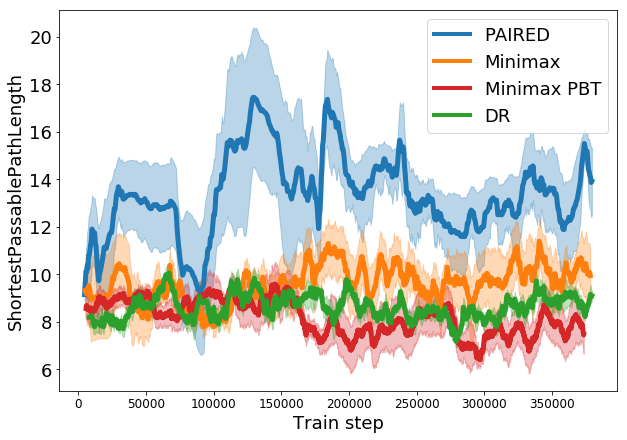}
  \caption{Passable path length}
\end{subfigure}%
\begin{subfigure}{.25\textwidth}
  \centering
  \includegraphics[width=\linewidth]{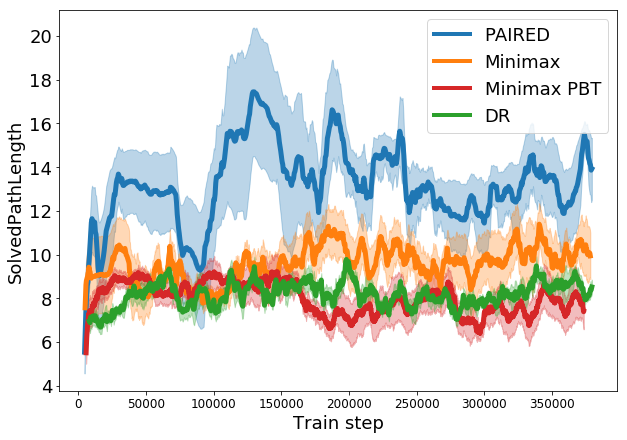}
  \caption{Solved path length}
  \label{fig:ex_fourrooms}
\end{subfigure}
\caption{Statistics of generated environments in terms of the number of blocks (a), distance from the start position to the goal (b), and the shortest path length between the start and the goal, which is zero if there is no possible path (c). The final plot shows agent learning in terms of the shortest path length of a maze successfully solved by the agents. Each plot is measured over five random seeds; error bars are a 95\% CI.
Domain randomization (DR) cannot tailor environment design to the agent's progress, so metrics remain fixed or vary randomly. Minimax training (even with populations of adversaries and agents) has no incentive to improve agent performance, so the length of mazes that agents are able to solve remains similar to DR (d). In contrast, PAIRED is the only method that continues to increase the passable path length to create more challenging mazes (c), producing agents that solve more complex mazes than the other techniques (d).}
\label{fig:complexity}
\end{figure}

\subsubsection{Emergent Complexity}
\label{sec:emergent_complexity}

Prior work \cite{wang2019paired,wang2020enhanced} focused on demonstrating emergent complexity as the primary goal, arguing that automatically learning complex behaviors is key to improving the sophistication of AI agents. Here, we track the complexity of the generated environments and learned behaviors throughout training.
Figure \ref{fig:complexity} shows the number of blocks (a), distance to the goal (b), and the length of the shortest path to the goal (c) in the generated environments. The solved path length (d) tracks the shortest path length of a maze that the agent has completed successfully, and can be considered a measure of the complexity of the agent's learned behavior.

Domain randomization (DR)
simply maintains environment parameters within a fixed range, and cannot continue to propose increasingly difficult tasks.
Techniques based on minimax training are purely motivated to decrease the protagonist's score, and as such do not enable the agent to continue learning; both minimax and PBT obtain similar performance to DR.
In contrast, PAIRED creates a curriculum of environments that begin with shorter paths and fewer blocks, but gradually increase in complexity based on both agents' current level of performance.  As shown in Figure \ref{fig:complexity} (d), this allows PAIRED protagonists to learn to solve more complex environments than the other two techniques.

\subsubsection{Zero-Shot Transfer}

\captionsetup[sub]{font=small}
\begin{figure}
\centering
\begin{subfigure}{.16\textwidth}
  \centering
  \includegraphics[width=\linewidth]{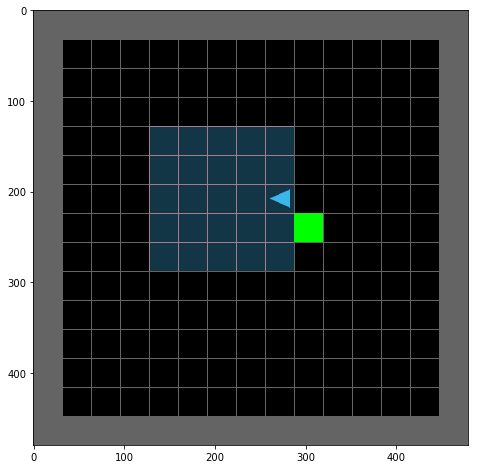}
\end{subfigure}%
\begin{subfigure}{.16\textwidth}
  \centering
  \includegraphics[width=\linewidth]{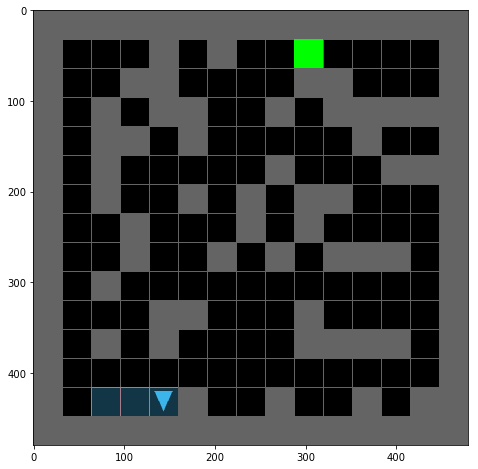}
\end{subfigure}%
\begin{subfigure}{.16\textwidth}
  \centering
  \includegraphics[width=\linewidth]{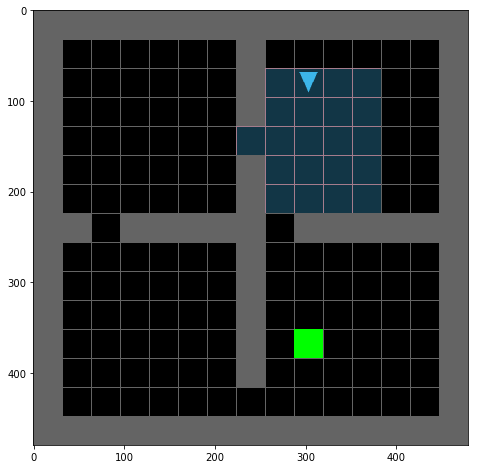}
\end{subfigure}%
\begin{subfigure}{.16\textwidth}
  \centering
  \includegraphics[width=\linewidth]{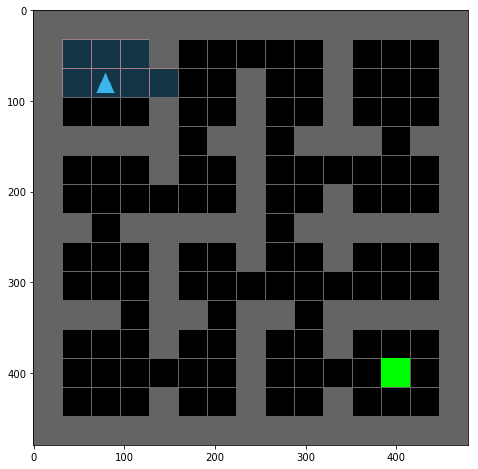}
\end{subfigure}%
\begin{subfigure}{.16\textwidth}
  \centering
  \includegraphics[width=\linewidth]{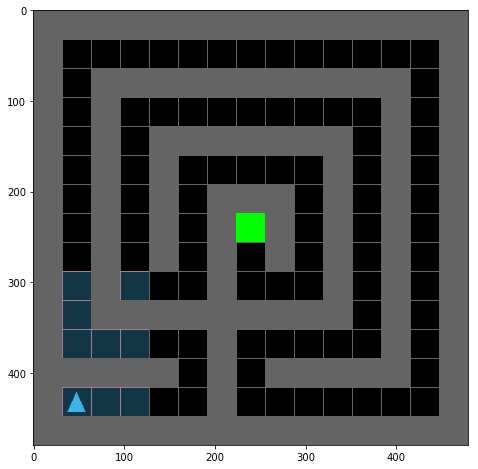}
\end{subfigure}%
\begin{subfigure}{.16\textwidth}
  \centering
  \includegraphics[width=\linewidth]{figures/transfer_envs/maze.png}
\end{subfigure}

\centering
\begin{subfigure}{.16\textwidth}
  \centering
  \includegraphics[width=\linewidth]{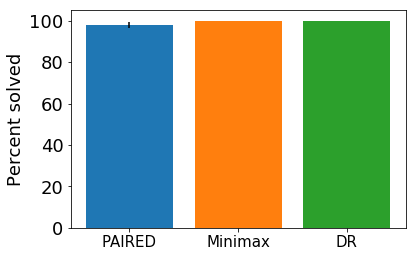}
  \caption{Empty}
\end{subfigure}%
\begin{subfigure}{.16\textwidth}
  \centering
  \includegraphics[width=\linewidth]{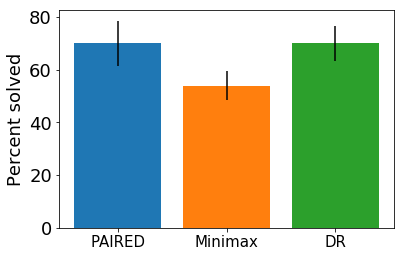}
  \caption{50 Blocks}
\end{subfigure}%
\begin{subfigure}{.16\textwidth}
  \centering
  \includegraphics[width=\linewidth]{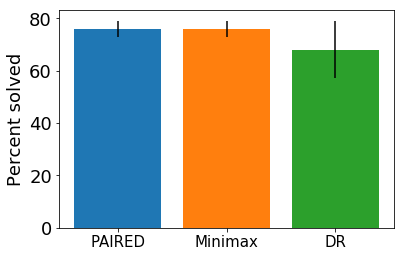}
  \caption{4 Rooms}
\end{subfigure}%
\begin{subfigure}{.16\textwidth}
  \centering
  \includegraphics[width=\linewidth]{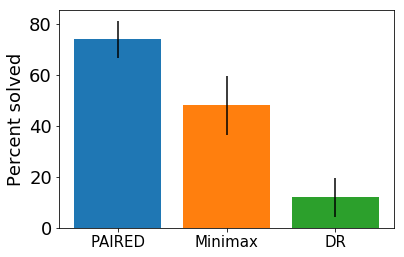}
  \caption{16 Rooms}
\end{subfigure}%
\begin{subfigure}{.16\textwidth}
  \centering
  \includegraphics[width=\linewidth]{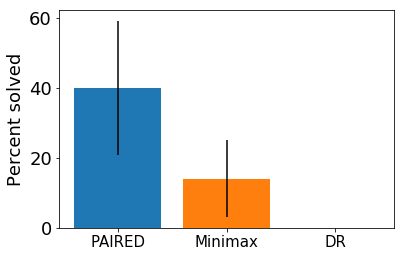}
  \caption{Labyrinth}
\end{subfigure}%
\begin{subfigure}{.16\textwidth}
  \centering
  \includegraphics[width=\linewidth]{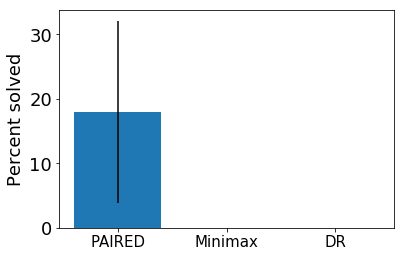}
  \caption{Maze}
\end{subfigure}
\caption{
Percent successful trials in environments used to test zero-shot transfer, out of 10 trials each for 5 random seeds. The first two (a, b) simply test out-of-distribution generalization to a setting of the number of blocks parameter. Four Rooms (c) tests within-distribution generalization to a specific configuration that is unlikely to be generated through random sampling. The 16 Rooms (d), Labyrinth environment (e) and Maze (f) environments were designed by a human to be challenging navigation tasks.
The bar charts show the zero-shot transfer performance of models trained with domain randomization (DR), minimax, or PAIRED in each of the environments. Error bars show a 95\% confidence interval. As task difficulty increases, only PAIRED retains its ability to generalize to the transfer tasks.}
\label{fig:transfer}
\end{figure}

In order for RL algorithms to be useful in the real world, they will need to generalize to novel environment conditions that their developer was not able to foresee. 
Here, we test the zero-shot transfer performance on a series of novel navigation tasks shown in Figure \ref{fig:transfer}. The first two transfer tasks simply use a parameter setting for the number of blocks that is out of the distribution (OOD) of the training environments experienced by the agents (analogous to the evaluation of \citep{pinto2017supervision}). We expect these transfer scenarios to be easy for all methods. We also include a more structured but still very simple Four Rooms environment, where the number of blocks is within distribution, but in an unlikely (though simple) configuration. As seen in the figure, this can usually be completed with nearly straight-line paths to the goal. To evaluate difficult transfer settings, we include the 16 rooms, Labyrinth, and Maze environments, which require traversing a much longer and more challenging path to the goal. This presents a much more challenging transfer scenario, since the agent must learn meaningful navigation skills to succeed in these tasks.

Figure \ref{fig:transfer} shows zero-shot transfer performance. As expected, all methods transfer well to the first two environments, although the minimax adversary is significantly worse in 50 Blocks. Varying the number of blocks may be a relatively easy transfer task, since partial observability has been shown to improve generalization performance in grid-worlds \cite{ye2020rotation}. As the environments increase in difficulty, so does the performance gap between PAIRED and the prior methods. In 16 Rooms, Labyrinth, and Maze, PAIRED shows a large advantage over the other two methods. In the Labyrinth (Maze) environment, PAIRED agents are able to solve the task in 40\% (18\%) of trials. In contrast, minimax agents solve 10\% (0.0\%) of trials, and DR agents solve 0.0\%. The performance of PAIRED on these tasks can be explained by the complexity of the generated environments; as shown in Figure \ref{fig:example_grids}c, the field of view of the agent (shaded blue) looks similar to what it might encounter in a maze, although the idea of a maze was never explicitly specified to the agent, and this configuration blocks is very unlikely to be generated randomly. These results suggest that PAIRED training may be better able to prepare agents for challenging, unknown test settings.

\subsection{Continuous Control Tasks}
\label{sec:cont_control}
\captionsetup[sub]{font=small}
\begin{figure}
\centering
\begin{subfigure}{.385\textwidth}
  \centering
  \includegraphics[width=\linewidth]{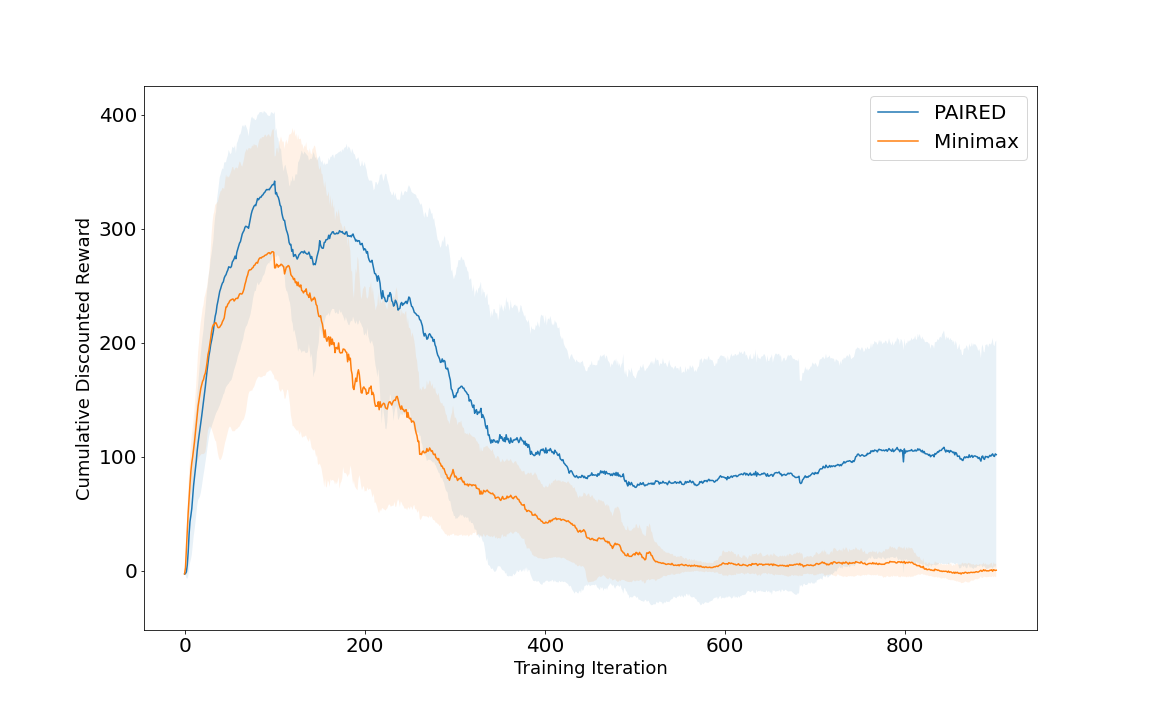}
  \caption{Reward}
  \label{fig:mujoco_rewards}
\end{subfigure}%
\begin{subfigure}{.3\textwidth}
  \centering
  \includegraphics[width=\linewidth]{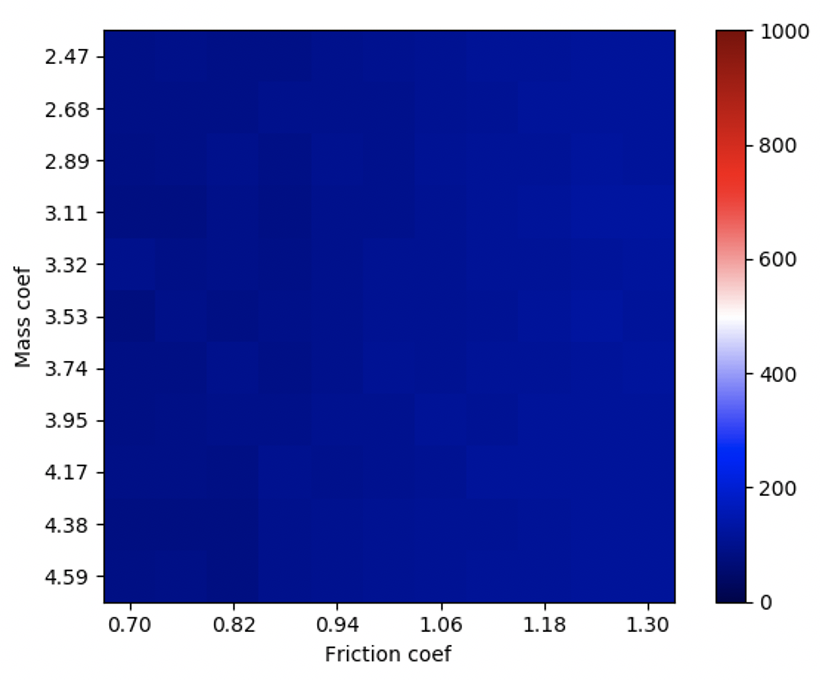}
  \caption{Minimax Adversarial}
\end{subfigure}%
\begin{subfigure}{.3\textwidth}
  \centering
  \includegraphics[width=\linewidth]{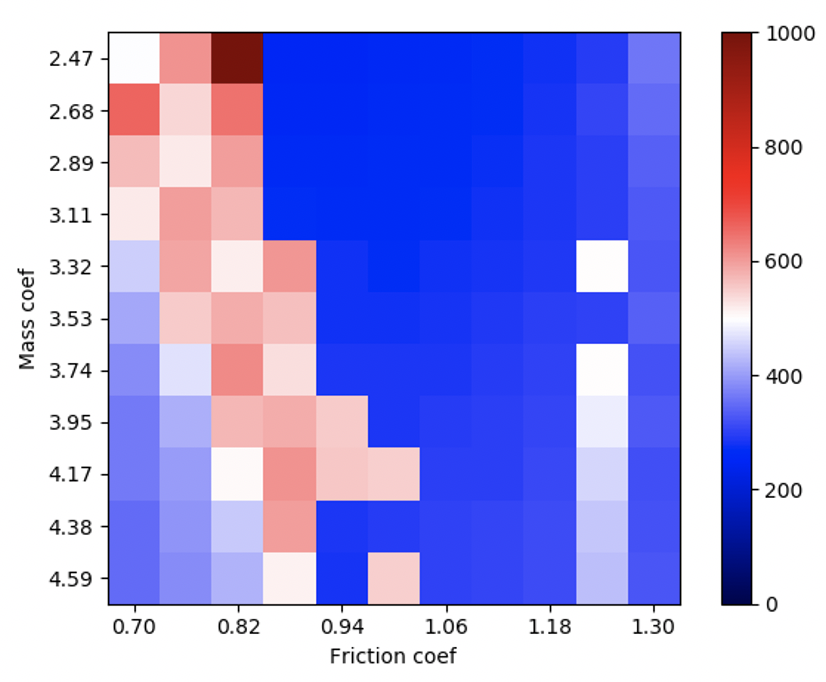}
  \caption{PAIRED}
\end{subfigure}%
\caption{Training curves for PAIRED and minimax adversaries on the MuJoCo hopper, where adversary strength is scaled up over the first 300 iterations (a). While both adversaries reduce the agent's reward by making the task more difficult, the minimax adversary is incentivized to drive the agent's reward as low as possible. In contrast, the PAIRED adversary must maintain feasible parameters. When varying force and mass are applied at test time, the minimax policy performs poorly in cumulative reward (b), while PAIRED is more robust (c).}
\label{fig:mujoco}
\vspace{-0.1cm}
\end{figure}

To compare more closely with prior work on minimax adversarial RL~\citep{pinto2017supervision,vinitsky2020dmalt}, we construct an additional experiment in a modified version of the MuJoCo hopper domain~\citep{todorov2012mujoco}. Here, the adversary outputs additional torques to be applied to each joint at each time step, $\vec{\theta}$. We benchmark PAIRED against unconstrained minimax training. To make minimax training feasible, the torque the adversary can apply is limited to be a proportion of the agent's strength, and is scaled from 0.1 to 1.0 over the course of 300 iterations. After this point, we continue training with no additional mechanisms for constraining the adversary. We pre-train the agents for 100 iterations, then test transfer performance to a set of mass and friction coefficients.  All agents are trained with PPO \citep{schulman2017proximal} and feed-forward policies.

Figure \ref{fig:mujoco_rewards} shows the rewards throughout training. We observe that after the constraints on the adversaries are removed after 300 iterations, the full-strength minimax adversary has the ability to make the task unsolvable, so the agent's reward is driven to zero for the rest of training. In contrast, PAIRED is able to automatically adjust the difficulty of the adversary to ensure the task remains solvable. As expected, Figure \ref{fig:mujoco} shows that when constraints on the adversary are not carefully tuned, policies trained with minimax fail to learn and generalize to unseen environment parameters. In contrast, PAIRED agents are more robust to unseen environment parameters, even without careful tuning of the forces that the adversary is able to apply.

\section{Conclusions}

We develop the framework of Unsupervised Environment Design (UED), 
and show its relevance to a range of RL tasks, from learning increasingly complex behavior or more robust policies, to improving generalization to novel environments.  In environments like games, for which the designer has an accurate model of the test environment, or can easily enumerate all of the important cases, using UED may be unnecessary.  However, UED could provide an approach to building AI systems in many of the ambitious uses of AI in real-world settings which are difficult to accurately model.

We have shown that tools from the decision theory literature can be used to characterize existing approaches to environment design, and motivate a novel approach based on minimax regret. Our algorithm, which we call Protagonist Antagonist Induced Regret Environment Design (PAIRED), 
avoids common failure modes of existing UED approaches, and is able to generate a curriculum of increasingly complex environments. Our results demonstrate that PAIRED agents learn more complex behaviors, and achieve higher zero-shot transfer performance in challenging, novel environments.

\section*{Broader Impact}
Unsupervised environment design is a technique with a wide range of applications, including unsupervised RL, transfer learning, and Robust RL.  Each of these applications has the chance of dramatically improving the viability of real-world AI systems.  Real-world AI systems could have a variety of positive impacts, reducing the possibility for costly human error, increasing efficiency, and doing tasks that are dangerous or difficult for humans.  However, there are a number of possible negative impacts: increasing unemployment by automating jobs \citep{frey2017future} 
and improving the capabilities of automated weapons.  These positives and negatives are potentially exacerbated by the emergent complexity we described in Section \ref{sec:emergent_complexity}, which could lead to improved efficiency and generality allowing the robotics systems to be applied more broadly.

However, if we are to receive any of the benefits of AI powered systems being deployed into real world settings, it is critical that we know how to make these systems robust and that they know how to make good decisions in uncertain environments.  In Section \ref{formalisms} we discussed the deep connections between UED and \emph{decisions under ignorance}, which shows that reasonable techniques for making decisions in uncertain settings correspond to techniques for UED, showing that the study of UED techniques can be thought of as another angle of attack at the problem of understanding how to make robust systems.  Moreover, we showed in Section \ref{sec:cont_control} that these connections are not just theoretical, but can lead to more effective ways of building robust systems.  Continued work in this area can help ensure that the predominant impact of our AI systems is the impact we intended.

Moreover, many of the risks of AI systems come from the system acting unexpectedly in a situation the designer had not considered.  Approaches for Unsupervised Environment Design could work towards a solution to this problem, by automatically generating interesting and challenging environments, hopefully detecting troublesome cases before they appear in deployment settings.

\begin{ack}
We would like to thank Michael Chang, Marvin Zhang, Dale Schuurmans, Aleksandra Faust, Chase Kew, Jie Tan, Dennis Lee, Kelvin Xu, Abhishek Gupta, Adam Gleave, Rohin Shah, Daniel Filan, Lawrence Chan, Sam Toyer, Tyler Westenbroek, Igor Mordatch, Shane Gu, DJ Strouse, and Max Kleiman-Weiner for discussions that contributed to this work. We are grateful for funding of this work as a gift from the Berkeley Existential Risk Intuitive.
We are also grateful to Google Research for funding computation expenses associated with this work.  
\end{ack}

\bibliographystyle{plainnat}
\bibliography{references}

\begin{appendix}

\section{Generality of UED}
\label{sec:theory}
\label{sec:universality}

In Section \ref{formalisms} we defined UPODMPS and UED and showed a selection of natural approaches to UED and corrisponding decision rules.  However, 
the connection between UED and decisions under ignorance is much broader than these few decision rules.  In fact, they can be seen as nearly identical problems. 

To see the connection to decisions under ignorance, it is important to notice that any decision rule can be thought of as an ordering over policies, ranking the policies that it chooses higher than other policies.  We will want to define a condition on this ordering, to do this we will first define what it means for one policy to \emph{totally dominate} another.

\begin{defn}
A policy, $\pi_A$, is \emph{totally dominated} by some policy, $\pi_B$ if for every pair of parameterizations $\vec{\theta}_A, \vec{\theta}_B$ $\Uf{\apply{\PPOMDP}{\vec{\theta}_A}}(\pi_A) < \Uf{\apply{\PPOMDP}{\vec{\theta}_B}}(\pi_B)$.
\end{defn} 

Thus if $\pi_A$ totally dominates $\pi_B$, it is reasonable to assume that $\pi_A$ is better, since the best outcome we could hope for from policy $\pi_B$ is still worse than the worst outcome we fear from policy $\pi_A$.  Thus we would hope that our decision rule would not prefer $\pi_B$ to $\pi_A$.  If a decision rule respects this property we will say that it $\emph{respects total domination}$ or formally:

\begin{defn}
We will say that an ordering $\prec$ \emph{respects total domination} iff $\pi_A \prec \pi_B$ whenever $\pi_B$ totally dominates $\pi_A$.
\end{defn}

This is a very weak condition, but it is already enough to allow us to provide a characterization of all such orderings over policies in terms of policy-conditioned distributions over parameters $\vec{\theta}$, which we will notate $\Lambda$ as in Section \ref{formalisms}.  Specifically, any ordering which respects total domination can be written as maximizing the expected value with respect to a \emph{policy-conditioned value function}, thus every reasonable way of ranking policies can be described in the UED framework.  For example, this implies that there is an environment policy which represents the strategy of minimax regret. We explicitly construct one such environment policy in Appendix \ref{minimax-regret-Lambda}.

To make this formal, the \emph{policy-conditioned value function}, $V^{\apply{\PPOMDP}{\Lambda}}(\pi)$, is defined to be the expected value a policy will receive in the policy-conditioned distribution of environments $\apply{\PPOMDP}{\Lambda(\pi)}$, or formally:
\begin{align}
V^{\apply{\PPOMDP}{\Lambda}}(\pi) = \EO\limits_{\vec{\theta} \sim \Lambda(\pi)}[\Uf{\vec{\theta}}(\pi)]
\end{align}.

The policy-conditioned value function is like the normal value function, but computed over the distribution of environments defined by the UPOMDP and environment policy. This of course implies an ordering over policies, defined in the natural way.

\begin{defn}
    We will say that $\Lambda$ prefers $\pi_B$ to $\pi_A$ notated $\pi_A \prec^\Lambda \pi_B$ if $V^{\apply{\PPOMDP}{\Lambda}}(\pi_A) < V^{\apply{\PPOMDP}{\Lambda}}(\pi_B)$
\end{defn}

Finally, this allows us to state the main theorem of this Section.

 \begin{theorem}
 \label{universality}
 	 Given an order over deterministic policies in a finite UPOMDP, $\prec$, there exits an environment policy $\Lambda: \Pi \rightarrow \Dist{\Theta^T}$ such that $\prec$ is equivalent to $\prec^\Lambda$ iff it respects total domination and it ranks policies with an equal and a deterministic outcome as equal.
 \end{theorem}
 
\begin{proof} 
    Suppose you have some order of policies in a finite UPOMDP, $\prec$, which respects total domination and ranks policies equal if they have an equal and deterministic outcome.  Our goal is to construct a function $\Lambda$ such that $\pi_A \prec^\Lambda \pi_B$ iff $\pi_A \prec \pi_B$.  
    
    When constructing $\Lambda$ notice that we can chose $\Lambda(\pi)$ independently for each policy $\pi$ and that we can choose the resulting value for $V^{\apply{\PPOMDP}{\Lambda}}(\pi)$ to lie anywhere within the range $[\min\limits_{\vec{\theta} \in \Theta^T}\{U^{\vec{\theta}}(\pi)\}, \max\limits_{\vec{\theta} \in \Theta^T}\{U^{\vec{\theta}}(\pi)\}]$.  Since the number of deterministic policies in a finite POMDP is finite, we can build $\Lambda$ inductively by taking the lowest ranked policy $\pi$ in terms of $\prec$ for which $\Lambda$ has not yet been defined and choosing the value for $V^{\apply{\PPOMDP}{\Lambda}}(\pi)$ appropriately.  
    
    For the lowest ranked policy $\pi_B$ for which $\Lambda(\pi_B)$ has not yet been defined we set $\Lambda(\pi_B)$ such that $V^{\apply{\PPOMDP}{\Lambda}}(\pi_B)$  greater than $\Lambda(\pi_A)$ for all $\pi_A \prec \pi_B$ and is lower than the minimum possible value of any $\pi_C$ such that $\pi_B \prec \pi_C$, if such a setting is possible.  That is, we choose $\Lambda(\pi_B)$ to satisfy for all $\pi_A \prec \pi_B \prec \pi_C$:
    \begin{align}
    \label{Theorem6condition}
   \Lambda(\pi_A) < V^{\apply{\PPOMDP}{\Lambda}}(\pi_B) < \max\limits_{\vec{\theta} \in \Theta^T}\{U^{\vec{\theta}}(\pi_C)\}
    \end{align}
    
    Intuitively this ensures that $V^{\apply{\PPOMDP}{\Lambda}}(\pi_B)$ is high enough to be above all $\pi_A$ lower than $\pi_B$ and low enough such that all future $\pi_C$ can still be assigned an appropriate value.
    
    Finally, we will show that it is possible to set $\Lambda(\pi_B)$ to satisfy these conditions in Equation \ref{Theorem6condition}.  By our inductive hypothesis, we know that $\Lambda(\pi_A) < \max\limits_{\vec{\theta} \in \Theta^T}\{U^{\vec{\theta}}(\pi_B)\}$ and $\Lambda(\pi_A) < \max\limits_{\vec{\theta} \in \Theta^T}\{U^{\vec{\theta}}(\pi_C)\}$.  Since $\prec$ respects total domination, we know that $\min\limits_{\vec{\theta} \in \Theta^T}\{U^{\vec{\theta}}(\pi_B)\} \leq \max\limits_{\vec{\theta} \in \Theta^T}\{U^{\vec{\theta}}(\pi_C)\}$ for all $\pi_B \leq \pi_C$.  Since there are a finite number of $\pi_C$ we can set $V^{\apply{\PPOMDP}{\Lambda}}(\pi_B)$ to be the average of the smallest value for  $\max\limits_{\vec{\theta} \in \Theta^T}\{U^{\vec{\theta}}(\pi_C)\}$ and the largest value for $\Lambda(\pi_A)$ for any $\pi_C$ and $\pi_A$ satisfying $\pi_A \prec \pi_B \prec \pi_C$.
    
    The other direction, can be checked directly. If $\pi_A$ is totally dominated by $\pi_B$, then: \[V^{\apply{\PPOMDP}{\Lambda}}(\pi_A) = \EO\limits_{\vec{\theta} \sim \Lambda(\pi_A)}[\Uf{\vec{\theta}}(\pi_A)] < \EO\limits_{\vec{\theta} \sim \Lambda(\pi_B)}[\Uf{\vec{\theta}}(\pi_B)] = V^{\apply{\PPOMDP}{\Lambda}}(\pi_B)\]
    
    Thus if $\prec$ can be represented with an appropriate choice of $\Lambda$ then it respects total domination and it ranks policies with an equal and a deterministic outcome equal.
\end{proof}

Thus, the set of decision rules which respect total domination are exactly those which can be written by some concrete environment policy, and thus any approach to making decisions under uncertainty which has this property can be thought of as a problem of unsupervised environment design and vice-versa.  To the best of our knowledge there is no seriously considered approach to making decisions under uncertainty which does not satisfy this property.

\section{Defining an Environment Policy Corresponding to Minimax Regret}
\label{minimax-regret-Lambda}

In this section we will be deriving a function $\Lambda(\pi)$ which corresponds to regret minimization. We will be assuming a finite MDP with a bounded reward function.  Explicitly, we will be deriving a $\Lambda$ such that minimax regret is the decision rule which maximizes the corresponding \emph{policy-conditioned value function} defined in Appendix \ref{sec:universality}: 

\[
V^{\apply{\PPOMDP}{\Lambda}}(\pi) = \EO\limits_{\vec{\theta} \sim \Lambda(\pi)}[\Uf{\vec{\theta}}(\pi)]
\]

In this context, we will define regret to be:

\[\normalfont\textsc{Regret}(\pi, \vec{\theta}) = \max\limits_{\pi^B \in \Pi}\{\Uf{\apply{\PPOMDP}{\vec{\theta}}}(\pi^B) - \Uf{\apply{\PPOMDP}{\vec{\theta}}}(\pi) \} \]

Thus the set of $\normalfont\textsc{MinimaxRegret}$ strategies can be defined naturally as the ones which achive the minimum worst case regret:

\[
 \normalfont\textsc{MinimaxRegret} =\argmin\limits_{\pi \in \Pi}\{\max\limits_{\vec{\theta} \in \Theta^T}\{\normalfont\textsc{Regret}(\pi,\vec{\theta})\}\}
\]

We want to find a function $\Lambda$ for which the set of optimal policies which maximize value with respect to $\Lambda$ is the same as the set of $\normalfont\textsc{MinimaxRegret}$ strategies.  That is, we want to find $\Lambda$ such that:
\[
\normalfont\textsc{MinimaxRegret} = \argmax_{\pi \in \Pi}\{V^{\apply{\PPOMDP}{\Lambda}}(\pi)\}
\]

To do this it is useful to first introduce the concept of \emph{weak total domination}, which is the natural weakening of the concept of total domination introduced in Appendix \ref{sec:universality}.  A policy is said to be weakly totally dominated by a policy $\pi_B$ if the maximum outcome that can be achieved by $\pi_A$ is equal to the minimum outcome that can be achieved by $\pi_B$, or formally:  

\begin{defn}
A policy, $\pi_A$, is \emph{weakly totally dominated} by some policy, $\pi_B$ if for every pair of parameterizations $\vec{\theta}_A, \vec{\theta}_B$ $\Uf{\apply{\PPOMDP}{\vec{\theta}_A}}(\pi_A) \leq \Uf{\apply{\PPOMDP}{\vec{\theta}_B}}(\pi_B)$.
\end{defn} 

The concept of weak total domination is only needed here to clarify what happens in the special case that there are policies that are weakly totally dominated but not totally dominated.  These policies may or may not be minimax regret optimal policies, and thus have to be treated with more care.  To get a general sense for how the proof works you may assume that there are no policies which are weakly totally dominated but not totally dominated and skip the sections of the proof marked as a special case: the second paragraph of the proof for Lemma \ref{lem:baseline_distribution} and the second paragraph of the proof for Theorem \ref{thm:miniamx_regret_environment_policy}.

We can use this concept to help define a normalization function $D: \Pi \rightarrow \Dist{\Theta^T}$ which has the property that if there are two policies $\pi_A,\pi_B$ such that neither is totally dominated by the other, then they evaluate the same value on the distribution of environments.  We will use $D$ as a basis for creating $\Lambda$ by shifting probability mass towards or away from this distribution in the right proportion.  In general there are many such normalization functions, but we will simply show that one of these exists.

\begin{lemma}
\label{lem:baseline_distribution}
    There exists a function $D: \Pi \rightarrow \Dist{\Theta^T}$ such that $\Lambda(\pi)$ has support at least $s > 0$ on the highest-valued regret-maximizing parameterization $\overline{\theta}_{\pi}$ for all $\pi$,  that for all $\pi_A, \pi_B$ such that neither is weakly totally dominated by any other policy, $V^{\apply{\PPOMDP}{D}}(\pi_A) = V^{\apply{\PPOMDP}{D}}(\pi_B)$ and $V^{\apply{\PPOMDP}{D}}(\pi_A) > V^{\apply{\PPOMDP}{D}}(\pi_B)$ when $\pi_B$ is totally dominated and $\pi_A$ is not.  If the policy $\pi$ is weakly dominated but not totally dominated we choose $D$ to put full support on the highest-valued regret-maximizing parameterization, $D(\pi) = \overline{\theta}_{\pi}$.
\end{lemma}
\begin{proof}
    We will first define $D$ for the set of policies which are not totally dominated, which we will call $X$.  Note that by the definition of not being totally dominated, there is a constant $C$ which is between the worst-case and best-case values of all of the policies in $X$. 
    
    \textbf{Special Case:} If there is only one choice for $C$ and if that does have support over $\overline{\theta}_{\pi}$ for each $\pi$ which is not totally dominated, then we know that there is a weakly totally dominated policy which is not totally dominated. In this case we choose a $C$ which is between the best-case and the worst-case for the not weakly totally dominated policies. Thus for each $\pi$ which is not weakly dominated we can find a distribution $\vec{\theta}$ such that $U^{\vec{\theta}}(\pi)=C$.  If C is not the maximum or minimum achievable value for $\pi$ then we will chose $D(\pi)$ to have expected value $C$ and and have support over all $\Theta^T$. 
    
    Otherwise, we can choose $D(\pi)=\overline{\theta}_{\pi}$.  For other $\pi \notin X$, we can let $D(\pi) = U(\Theta^T)$, which achieves value less than $C$ since all outcomes for $\pi$ have utility less than $C$.  Thus, by construction $D$ satisfies the desired conditions.
\end{proof}

Given such a function $D$, we will construct a $\Lambda$ which works by shifting probability towards or away from the environment parameters that maximize the regret for the given agent, in the right proportions to achieve the desired result.  We claim that the following definition works:
\[
\Lambda(\pi) = \{\overline{\theta}_{\pi}: \frac{c_\pi}{v_\pi}, \tilde{D}_{\pi}:\text{otherwise} \}
\]
Where the bracket notation defines a mixture distribution $\{x:P, y:Q\}(\theta) = x P(\theta)+y Q(\theta)$ and where $\tilde{D}_{\pi}=D(\pi)$ is a baseline distribution satisfying the conditions of Lemma~\ref{lem:baseline_distribution}, $\overline{\theta}_{\pi}$ is the trajectory which maximizes regret of $\pi$, and $v_\pi$ is the value above the baseline distribution that $\pi$ achieves on that trajectory, and $c_\pi$ is the negative of the worst-case regret of $\pi$, normalized so that $\frac{c_\pi}{v_\pi}$ is between $-s$ and $1$.  If $\frac{c_\pi}{v_\pi}$ is negative then we will interpret probability mass of $\frac{c_\pi}{v_\pi}$ on $\overline{\theta}_{\pi}$ in $\tilde{D}_{\pi}$ being redistributed proportionally in $\tilde{D}_{\pi}$ as is implied by the equation.  Note that if $v_\pi=0$ then $\tilde{D}_{\pi} =\overline{\theta}_{\pi} $ by our construction, so $\frac{c_\pi}{v_\pi}$ can be interpreted to be anything in $[0,1]$ without effecting the distribution $\Lambda(\pi)$.

 \begin{theorem}
 	 For $\Lambda$ as defined above: \[\argmax_{\pi \in \Pi}\{\EO\limits_{\vec{\theta} \sim \Lambda(\pi)}[U^{\vec{\theta}}(\pi)] \} = \normalfont\textsc{MinimaxRegret}\]
 	 \label{thm:miniamx_regret_environment_policy}
 \end{theorem}
\begin{proof}
    By the construction of $D(\pi)$ by Lemma~\ref{lem:baseline_distribution}, we will let $C= \EO\limits_{\vec{\theta} \sim \tilde{D}_{\pi}}[U^{\vec{\theta}}(\pi)]$, which is a constant independent of $\pi$ by construction for all of the policies which are not weakly totally dominated.  If there are no policies which are weakly totally dominated but not totally dominated the next paragraph can be skipped, otherwise it handles the special case by showing that the conditions we ensured for the non-totally dominated policies hold for these weakly totally dominated policies if and only if they are minimax regret optimal policies.
    
    \textbf{Special Case:} For the $\pi$ which are weakly totally dominated, but not totally dominated, we will show that $ \EO\limits_{\vec{\theta} \sim \tilde{D}_{\pi}}[U^{\vec{\theta}}(\pi)]=C$ if $\pi\in\normalfont\textsc{MinimaxRegret}$ and $ \EO\limits_{\vec{\theta} \sim \tilde{D}_{\pi}}[U^{\vec{\theta}}(\pi)]<C$ if $\pi \notin \normalfont\textsc{MinimaxRegret}$.  Let $M$ be the maximum value achieved by any policy. Then if $\pi\in\normalfont\textsc{MinimaxRegret}$ and the maximum value it achieves is $C'$ then note that any policy that weakly dominates $\pi$, say $\pi'$, can achieve no more than $M-C'$ regret. Thus if $\pi$ is a minimax optimal strategy, it must achieve exactly $M-C'$ regret and both $\pi$ and $\pi'$ must have a regret maximizing outcome at the same value $M-C'$, which is the maximum value outcome for $\pi$ and the minimum value outcome for $\pi'$.  Thus $C=C'$ since both $\pi$ and $\pi'$ must have support on the an outcome of value $C'$ and $\EO\limits_{\vec{\theta} \sim \tilde{D}_{\pi}}[U^{\vec{\theta}}(\pi)]=C$.  Otherwise $\pi \notin \normalfont\textsc{MinimaxRegret}$ so it achieves more than $M-C'$ regret and $\EO\limits_{\vec{\theta} \sim \tilde{D}_{\pi}}[U^{\vec{\theta}}(\pi)]<C$.
    
    With this case out of the way, we can use the definition of $\Lambda$ and simplify to show the desired result.
    \begin{align}
        \argmax_{\pi \in \Pi}\{\EO\limits_{\vec{\theta} \sim \Lambda(\pi)}[U^{\vec{\theta}}(\pi)]\} &= \argmax_{\pi \in \Pi}\{\frac{U^{\overline{\theta}_{\pi}}(\pi)c_\pi}{v_\pi}+C(1-\frac{c_\pi}{v_\pi})\}\\
        &= \argmax_{\pi \in \Pi}\{\frac{v_\pi c_\pi + Cc_\pi}{v_\pi}+(\frac{Cv_\pi -Cc_\pi}{v_\pi})\}\\
        &= \argmax_{\pi \in \Pi}\{\frac{v_\pi c_\pi+ C v_\pi}{v_\pi}\}\\
        &= \argmax_{\pi \in \Pi}\{c_\pi+C\}\\
        &= \argmax_{\pi \in \Pi}\{c_\pi\}\\
        &= \argmax_{\pi \in \Pi}\{\min\limits_{\vec{\theta}\in \Theta^T}\{-\normalfont\textsc{Regret}(\pi,\vec{\theta})\}\}\\
        &= \argmin_{\pi \in \Pi}\{\max\limits_{\vec{\theta}\in \Theta^T}\{\normalfont\textsc{Regret}(\pi,\vec{\theta})\}\}
    \end{align}
\end{proof}

\section{Minimax Regret Always Succeeds when There is a Clear Notion of Success}
In this section we will show that when there is a sufficiently strong notion of success and failure, and there is a policy which can ensure success, minimax regret will choose a successful strategy.  Moreover, we will show that neither pure randomization, nor maximin has this property.

\begin{customthm}{\ref{success:theorem}}
 Suppose that all achievable rewards fall into one of two class of outcomes labeled \textbf{SUCCESS} giving rewards in $[\textbf{S}_{min}, \textbf{S}_{max}]$ and \textbf{FAILURE} giving rewards in $[\textbf{F}_{min}, \textbf{F}_{max}]$, such that $\textbf{F}_{min} \leq \textbf{F}_{max} < \textbf{S}_{min}\leq \textbf{S}_{max}$.  In addition assume that the range of possible rewards in either class is smaller than the difference between the classes so we have $\textbf{S}_{max} - \textbf{S}_{min} < \textbf{S}_{min} - \textbf{F}_{max}$ and $ \textbf{F}_{max} - \textbf{F}_{min} < \textbf{S}_{min} - \textbf{F}_{max}$. Further suppose that there is a policy $\pi$ which succeeds on any $\vec{\theta}$ whenever success is possible. Then minimax regret will choose a policy which has that property.
 
 Further suppose that there is a policy $\pi$ which succeeds on any $\vec{\theta}$ whenever any policy succeeds on $\vec{\theta}$. Then minimax regret will choose a policy which has that property.
\end{customthm}
\begin{proof}
     Let $C = \textbf{S}_{min} - \textbf{F}_{max}$ and let $\pi^*$ be one of the policies that succeeds on any $\vec{\theta}$ whenever success is possible.  By assumption $\textbf{S}_{max} - \textbf{S}_{min} < C$ and $ \textbf{F}_{max} - \textbf{F}_{min} < C$.  Since $\pi^*$ succeeds whenever success is possible we have:
     \[
     \max\limits_{\vec{\theta}\in \Theta^T}\{\normalfont\textsc{Regret}(\pi^* ,\vec{\theta})\} < C
     \]
    
     Thus any strategy which achieves minimax regret must have regret less than $C$, since it must do at least as well as $\pi^*$.  Suppose $\pi \in \normalfont\textsc{MinimaxRegret}$, then if $\pi$ does not solve some solvable $\vec{\theta}$ then $\normalfont\textsc{Regret}(\pi^*,\vec{\theta}) > C$.  Thus every minimax regret policy succeeds whenever that is possible.
\end{proof}

Though minimax regret has this property, the other decision rules considered do not.  For example, consider the task described by Table \ref{small_game}, where an entry in position $(i,j)$ shows the payoff $U^{\theta_j}(\pi_i)$ obtained for the policy $\pi_i$ in row $i$, in the environment parameterized by $\theta_j$ in column $j$. If you choose $\pi_B$ you can ensure success whenever any policy can ensure success, but maximin will choose $\pi_A$, as it only cares about the worst case outcome.
\begin{table}[h!]
\begin{subtable}{.4\linewidth}
\begin{center}
\begin{tabular}{ c | c  | c }
        & $\theta_1$ & $\theta_2$ \\
 \hline
 $\pi_A$ & 0 &  0 \\ 
  \hline
 $\pi_B$ & 100 & -1    
\end{tabular}
\caption{}
\label{small_game}
\end{center}
\end{subtable}%
\begin{subtable}{.4\linewidth}
\begin{center}
\begin{tabular}{ c | c  | c | c | c | c}
        & $\theta_1$ & $\theta_2$ & $\theta_3$ & $\theta_4$ & $\theta_5$  \\
 \hline
 $\pi_A$ & 75 &  75 & 75 &  75 & 75 \\ 
  \hline
 $\pi_B$ & 0 & 100 & 100 & 100 & 100  \\
 \hline
 $\pi_C$ & 100 & 0 & 100 & 100 & 100  \\
 \hline
 $\pi_D$ & 100 & 100 & 0 & 100 & 100  \\
 \hline
 $\pi_E$ & 100 & 100 & 100 & 0 & 100  \\
 \hline
 $\pi_F$ & 100 & 100 & 100 & 100 & 0  \\
\end{tabular}
\caption{}
\label{big_game}
\end{center}
\end{subtable}
\caption{In these setting, \textbf{SUCCESS} is scoring in the range $[75,100]$ and \textbf{FAILURE} is scoring in the range $[-1,0]$}
\end{table}

  Choosing the policy that maximizes the expected return fails in the game described by Table \ref{big_game}.  Under a uniform distribution all of $\pi_B,\pi_C,\pi_D,\pi_E,\pi_F$ give an expected value of $80$, while $\pi_A$ gives an expected value of $75$.  Here maximizing expected value on a uniform distribution will choose one of $\pi_B,\pi_C,\pi_D,\pi_E,\pi_F$ while $\pi_A$ guarantees success.  Moreover, this holds for every possible distribution over parameters, since $\pi_A$ will always give an expected value of $75$ and on average $\pi_B,\pi_C,\pi_D,\pi_E,\pi_F$ will give an expected value of $80$, so at least one of them gives above $80$.

\section{Nash solutions to PAIRED}
 \label{Nash_PAIRED}
In this section we show how PAIRED works when coordination is not achieved.  The following proof shows that if the antagonist, protagonist, and adversary find a Nash Equilibrium, then the protagonist performs better than or equal to the antagonist in every parameterization.

\begin{theorem}
    Let $(\pi^P,\pi^A,\vec{\theta})$ be in Nash equilibria. Then for all $\vec{\theta}'$:  $\Uf{\vec{\theta}'}(\pi_P) \geq \Uf{\vec{\theta}'}(\pi_A)$.
\end{theorem}
\begin{proof}
Let $(\pi^P,\pi^A,\vec{\theta})$ be in Nash equilibria.  Then regardless of $\pi^A,\vec{\theta}$, the protagonist always has the choice to play $\pi^P = \pi^A$, so $\normalfont\textsc{Regret}(\pi^P, \pi^A, \vec{\theta}) \leq 0$.  Since, the adversary chooses $\vec{\theta}$ over any other $\vec{\theta}'$ we know that $\normalfont\textsc{Regret}(\pi^P, \pi^A, \vec{\theta}') \leq 0$ must hold for all $\vec{\theta}'$. Thus by the definition of Regret $\Uf{\vec{\theta}'}(\pi_P) \geq \Uf{\vec{\theta}'}(\pi_A)$. 
\end{proof}

This shows that with a capable antagonist the protagonist would learn the minimax regret policy, even without the adversary and the antagonist coordinating, and suggests that methods which strengthen the antagonist could serve to improve the protagonist.

\section{Additional experiments}
\subsection{Alternative methods for computing regret}
\label{sec:app_combined}
Given the arguments in Appendix \ref{Nash_PAIRED}, we experimented with additional methods for approximating the regret, which attempt to improve the antagonist. We hypothesized that using a fixed agent as the antagonist could potentially lead to optimization problems in practice; if the antagonist became stuck in a local minimum and stopped learning, it could limit the complexity of the environments the adversary could generate. Therefore, we developed an alternative approach using population-based training (PBT), where for each environment designed by the adversary, each agent in the population collects a single trajectory in the environment. Then, the regret is computed using the difference between the maximum performing agent in the population, and the mean of the population. Assume there is a population of $K$ agents, and $i$ and $j$ index agents within the population. Then: 

\begin{align}
    \textsc{Regret}_{pop} = \max_{i} U(\tau^i) - \frac{1}{K} \sum_{j=1}^K [U(\tau^j)]
\end{align}

We also used a population of adversaries, where for each episode, we randomly select an adversary to generate the environment, then test all agents within that environment. We call this approach \textit{PAIRED Combined PBT}.

Since using PBT can be expensive, we also investigated using a similar approach in the case of a population of $K=2$ agents, and a single adversary. This is analogous to a version of PAIRED where there is no fixed antagonist, but the antagonist is flexibly chosen to be the currently best-performing agent. We call this approach \textit{Flexible PAIRED}. 

Figure \ref{fig:combined_flexible} plots the complexity results of these two approaches. We find that they both perform worse than the proposed version of PAIRED. Figure \ref{fig:combined_flexible_transfer} shows the transfer performance of both methods against the original PAIRED. Both methods achieve reasonable transfer performance, retaining the ability to solve complex test tasks like the labyrinth and maze (when domain randomization and minimax training cannot). However, we find that the original method for computing the regret outperforms both approaches. We note that both the combined population and flexible PAIRED approaches do not enable the adversary and antagonist to coordinate (since the antagonist does not remain fixed). Coordination between the antagonist and adversary is an assumption in Theorem 2 of Section \ref{sec:paired}, which shows that if the agents reach a Nash equilibrium, the protagonist will be playing the minimum regret policy. These empirical results appear to indicate that when coordination between the adversary and antagonist is no longer possible, performance degrades. Future work should investigate whether these results hold across more domains.

\captionsetup[sub]{font=small}
\begin{figure}
\centering
\begin{subfigure}{.25\textwidth}
  \centering
  \includegraphics[width=\linewidth]{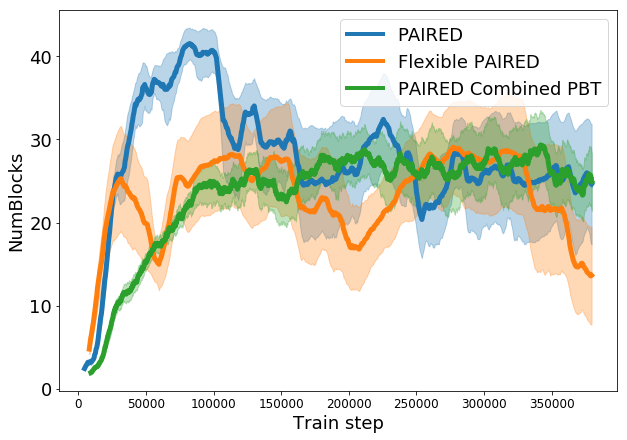}
  \caption{Number of blocks}
\end{subfigure}%
\begin{subfigure}{.25\textwidth}
  \centering
  \includegraphics[width=\linewidth]{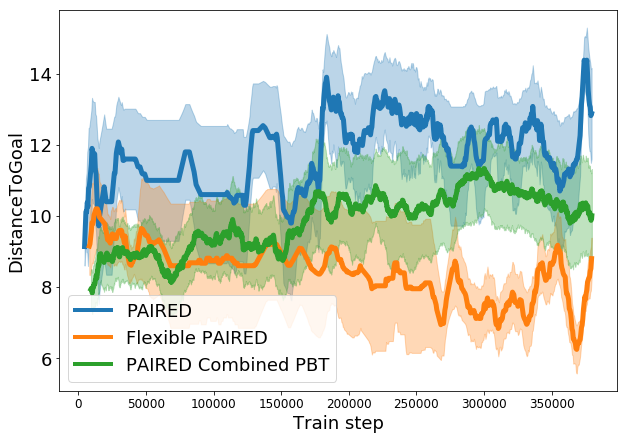}
  \caption{Distance to goal}
\end{subfigure}%
\begin{subfigure}{.25\textwidth}
  \centering
  \includegraphics[width=\linewidth]{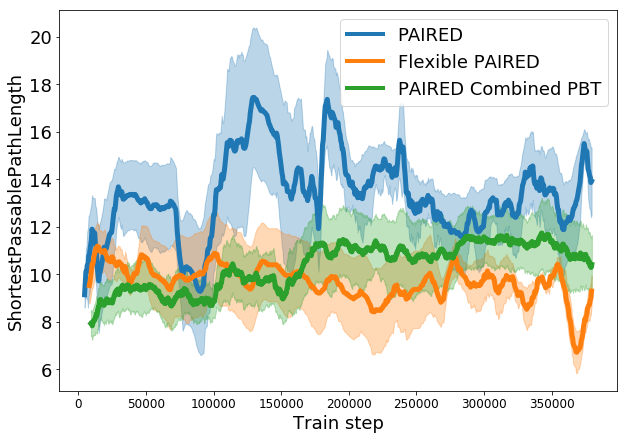}
  \caption{Passable path length}
\end{subfigure}%
\begin{subfigure}{.25\textwidth}
  \centering
  \includegraphics[width=\linewidth]{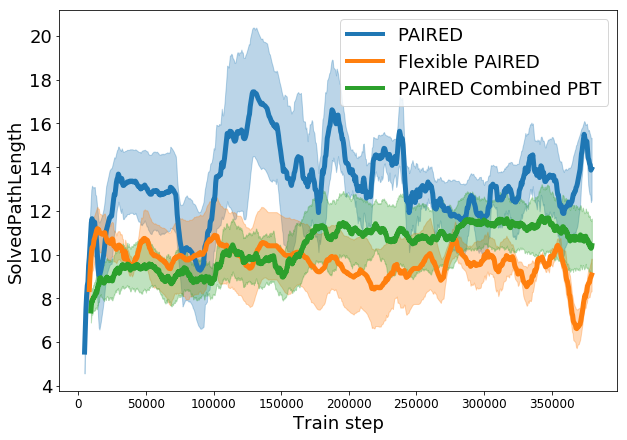}
  \caption{Solved path length}
\end{subfigure}
\caption{Comparison of alternative methods for computing the regret in terms of the statistics of generated environments (a-c), and the length of mazes that agents are able to solve (d). Neither the combined population or flexible approach (which both break the coordination assumption in Theorem 2) show improved complexity. Statistics of generated and solved environments measured over five random seeds; error bars are a 95\% CI.}
\label{fig:combined_flexible}
\end{figure}

\begin{figure}
\centering
\begin{subfigure}{.16\textwidth}
  \centering
  \includegraphics[width=\linewidth]{figures/transfer_envs/empty.png}
\end{subfigure}%
\begin{subfigure}{.16\textwidth}
  \centering
  \includegraphics[width=\linewidth]{figures/transfer_envs/cluttered50.png}
\end{subfigure}%
\begin{subfigure}{.16\textwidth}
  \centering
  \includegraphics[width=\linewidth]{figures/transfer_envs/fourrooms.png}
\end{subfigure}%
\begin{subfigure}{.16\textwidth}
  \centering
  \includegraphics[width=\linewidth]{figures/transfer_envs/sixteenroomsfewerdoors.png}
\end{subfigure}%
\begin{subfigure}{.16\textwidth}
  \centering
  \includegraphics[width=\linewidth]{figures/transfer_envs/labyrinth.png}
\end{subfigure}%
\begin{subfigure}{.16\textwidth}
  \centering
  \includegraphics[width=\linewidth]{figures/transfer_envs/maze.png}
\end{subfigure}

\begin{subfigure}{.16\textwidth}
  \centering
  \includegraphics[width=\linewidth]{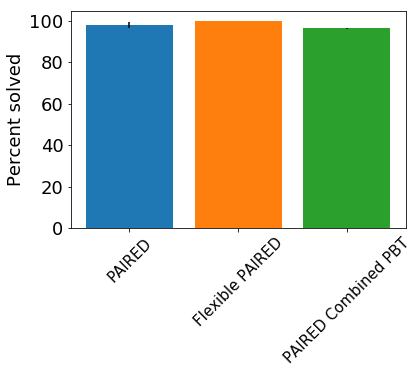}
  \caption{Empty}
\end{subfigure}%
\begin{subfigure}{.16\textwidth}
  \centering
  \includegraphics[width=\linewidth]{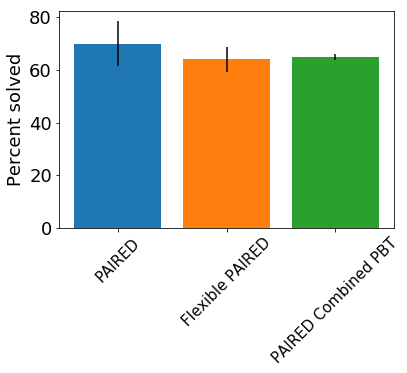}
  \caption{50 Blocks}
\end{subfigure}%
\begin{subfigure}{.16\textwidth}
  \centering
  \includegraphics[width=\linewidth]{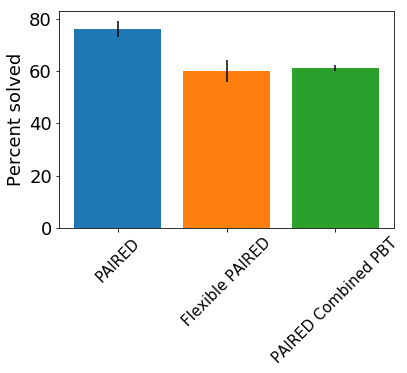}
  \caption{4 Rooms}
\end{subfigure}%
\begin{subfigure}{.16\textwidth}
  \centering
  \includegraphics[width=\linewidth]{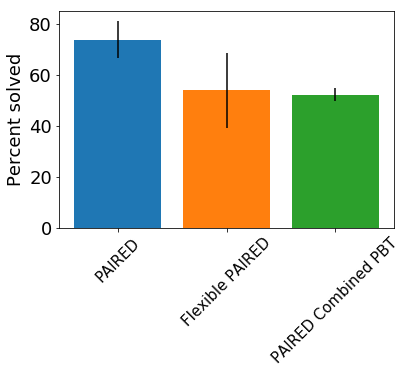}
  \caption{16 Rooms}
\end{subfigure}%
\begin{subfigure}{.16\textwidth}
  \centering
  \includegraphics[width=\linewidth]{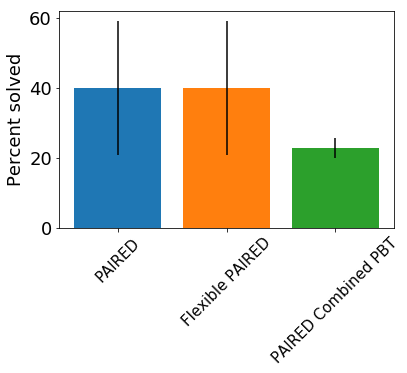}
  \caption{Labyrinth}
\end{subfigure}%
\begin{subfigure}{.16\textwidth}
  \centering
  \includegraphics[width=\linewidth]{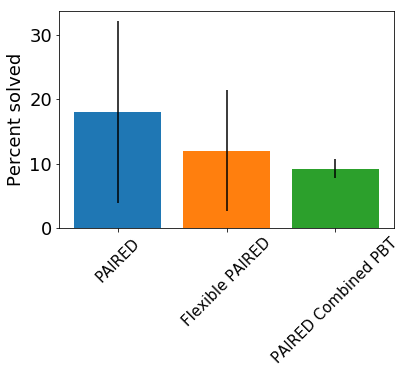}
  \caption{Maze}
\end{subfigure}
\caption{
Comparison of transfer performance for alternative methods for computing the regret. While both alternatives still retain the ability to solve complex tasks, achieving superior performance to minimax and domain randomization, their performance is not as good as the proposed method for computing the regret.}
\label{fig:combined_flexible_transfer}
\end{figure}

\subsection{Should agents themselves optimize regret?}
We experimented with whether the protagonist and antagonist should optimize for regret, or for the normal reward supplied by the environment (note that the environment-generating adversary always optimizes for the regret). For both the protagonist and antagonist, the regret is based on the difference between their reward for the current trajectory, and the max reward of the other agent received over several trajectories played in the current environment. Let the current agent be $A$, and the other agent be $O$. Then agent $A$ should receive reward $R^A = U(\tau^A) - \max_{\tau^O} U(\tau^O)$ for episode $\tau^A$. However, this reward is likely to be negative. If given at the end of the trajectory, it would essentially punish the agent for reaching the goal, making learning difficult. Therefore we instead compute a per-timestep penalty of $\frac{\max_{\tau^O} U(\tau^O)}{T}$, where $T$ is the maximum length of an episode. This is subtracted post-hoc from each agent's reward at every step during each episode, after the episode has been collected but before the agent is trained on it. When the agent reaches the goal, it receives the normal Minigrid reward for successfully navigating to the goal, which is $R = 1 - 0.9 * (M / T)$, where $M$ is the timestep on which the agent found the goal. 

Figure \ref{fig:regret_v_not} shows the complexity results for the best performing hyperparameters in which the protagonist and antagonist optimized the regret according to the formula above, or whether they simply learned according to the environment reward. Note that in both cases, the environment-generating adversary optimizes the regret of the protagonist with respect to the antagonist. As is evident in Figure \ref{fig:regret_v_not}, training agents on the environment reward itself, rather than the regret, appears to be more effective for learning complex behavior. We hypothesize this is because the regret is very noisy. The performance of the other agent is stochastic, and variations in the other agent's reward are outside of the agent's control. Further, agents do not receive observations about the other agent, and cannot use them to determine what is causing the reward to vary. However, we note that optimizing for the regret can provide good transfer performance. The transfer plots in Figure \ref{fig:transfer} were created with an agent that optimized for regret, as we describe below. It is possible that as the other agent converges, the regret provides a more reliable signal indicating when the agent's performance is sub-optimal.

\begin{figure}
\centering
\begin{subfigure}{.25\textwidth}
  \centering
  \includegraphics[width=\linewidth]{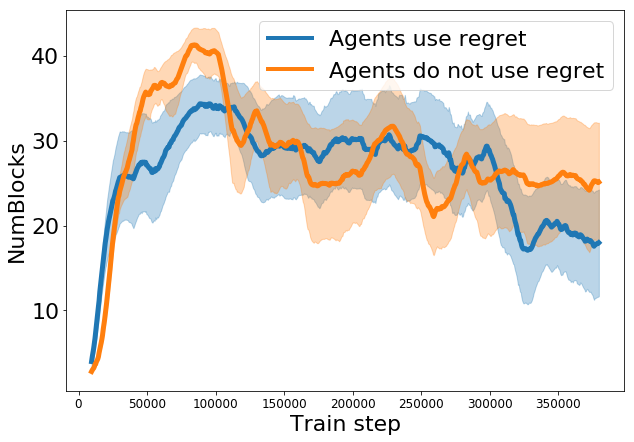}
  \caption{Number of blocks}
\end{subfigure}%
\begin{subfigure}{.25\textwidth}
  \centering
  \includegraphics[width=\linewidth]{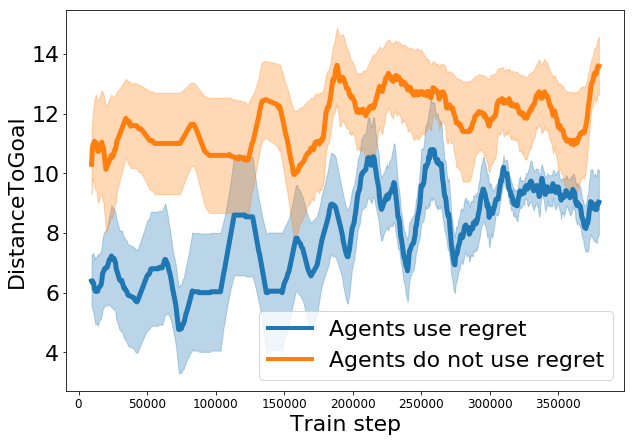}
  \caption{Distance to goal}
\end{subfigure}%
\begin{subfigure}{.25\textwidth}
  \centering
  \includegraphics[width=\linewidth]{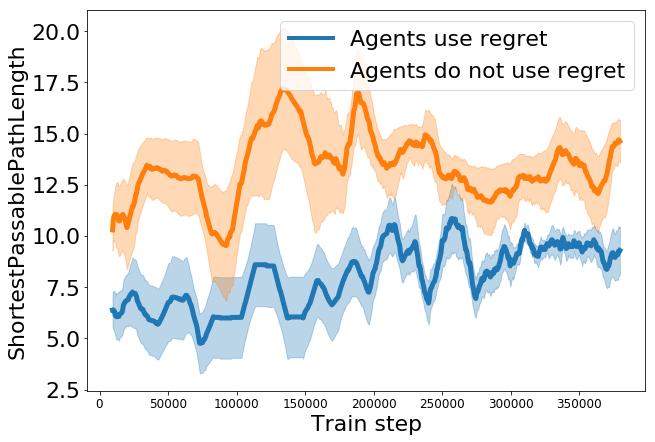}
  \caption{Passable path length}
\end{subfigure}%
\begin{subfigure}{.25\textwidth}
  \centering
  \includegraphics[width=\linewidth]{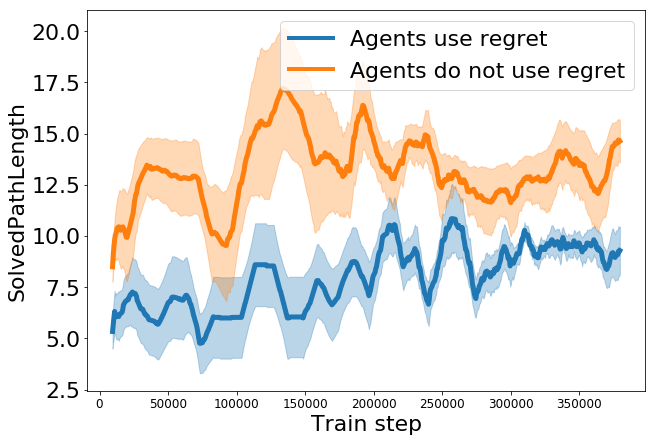}
  \caption{Solved path length}
\end{subfigure}
\caption{Comparison of alternative methods for computing the regret in terms of the statistics of generated environments (a-c), and the length of mazes that agents are able to solve. Neither the combined population or flexible approach show improved complexity. Statistics of generated and solved environments measured over five random seeds; error bars are a 95\% CI.}
\label{fig:regret_v_not}
\end{figure}

\section{Experiment Details and Hyperparameters}
\label{sec:app_hparams}

\subsection{Navigation Experiments}
\subsubsection{Agent parameterization}
The protagonist and antagonist agents for both PAIRED and the baselines received a partially observed, $5\times5\times3$ view of the environment, as well as an integer ranging from $0-3$ representing the direction the agent is currently facing. The agents use a network architecture consisting of a single convolutional layer which connects to an LSTM, and then to two fully connected layers which connect to the policy outputs. A second network with identical parameters is used to estimate the value function. The agents use convolutional kernels of size $3$ with $16$ filters to input the view of the environment, an extra fully-connected layer of size $5$ to process the direction and input it to the LSTM, an LSTM of size $256$, and two fully connected layers of size $32$ which connect to either the policy outputs or the value estimate. The best entropy regularization coefficient was found to be $0.0$. All agents (including environment adversaries) are trained with PPO with a discount factor of $0.995$, a learning rate of $0.0001$, and $30$ environment workers operating in parallel to collect a batch of episodes, which is used to complete one training update.

\subsubsection{Adversary parameterization}
The environments explored in this paper are a $15\times15$ tile discrete grid, with a border of walls around the edge. This means there are $13\times13=169$ free tiles for the adversary to use to place obstacles. We parameterized the adversary by giving it an action space of dimensionality 169, and each discrete action indicates the location of the next object to be placed. It plays a sequence of actions, such that on the first step it places the agent, on the second it places the goal, and for $50$ steps afterwards it places a wall (obstacle). If the adversary places an object on top of a previously existing object, its action does nothing; this allows it to place fewer than $50$ obstacles. If it tries to place the goal on top of the agent, the goal will be placed randomly.

We also explored an alternative parameterization, in which the adversary had only $4$ actions, which corresponded to placing the agent, goal, an obstacle, or nothing. It then took a sequence of $169$ steps to place all objects, moving through the grid from top to bottom and left to right. If it chose to place the agent or goal when they had already been placed elsewhere in the map, they would be moved to the current location. This parameterization allows the adversary to place as many blocks as there are squares in the map. However, we found that adversaries trained with this parameterization drastically underperformed the alternative version used in the paper, scoring an average solved path length of $\approx2$, as opposed to $\approx15$. We hypothesize this is because when the adversary is randomly initialized, sampling from its random policy is more likely to produce impossible environments. This makes it impossible for the agents to learn, and cannot provide a regret signal to the adversary to allow it to improve. We suggest that when designing an adversary parameterization, it may be important to ensure that sampling from a random adversary policy can produce feasible environments. 

The environment-constructing adversary's observations consist of a $15\times15\times3$ image of the state of the environment, an integer $t$ representing the current timestep, and a random vector $z \sim \mathcal{N}(0,I), z \in \mathbb{R}^{50}$ to allow it to generate random mazes. Because the sequencing of actions is important, we experiment with using an RNN to parameterize the adversary as well, although we find it is not always necessary. The adversary architecture is similar to that of the agents; it consists of a single convolutional layer which connects to an LSTM, and then to two fully connected layers which connect to the policy outputs. Additional inputs such as $t$ and $z$ are connected directly to the LSTM layer. We use a second, identical network to estimate the value function. 
To find the best hyperparameters for PAIRED and the baselines, we perform a grid search over the number of convolution filters used by the adversary, the degree of entropy regularization, which of the two parameterizations to use, and whether or not to use an RNN, and finally, the number of steps the protagonist is given to find the goal (we reasoned that lowering the protagonist episode length relative to the antagonist length would make the regret a less noisy signal). These parameters are shared between PAIRED and the baseline Minimax adversary, and we sweep all of these parameters equally for both. For PAIRED, we also experiment with whether the protagonist and antagonist optimize regret, and with using a non-negative regret signal, \textit{i.e.} $\normalfont\textsc{Regret} = \max(0, \normalfont\textsc{Regret})$, reasoning this could also lead to a less noisy reward signal.

For the complexity experiments, we chose the parameters that resulted in the highest solved path length in the last 20\% of training steps. The best parameters for PAIRED were an adversary with $128$ convolution filters, entropy regularization coefficient of $0.0$, protagonist episode length of $250$, non-negative regret, and the agents did not optimize regret. The best parameters for the minimax adversary were $256$ convolution filters, entropy regularization of $0.0$, and episode length of $250$. For the Population-based-training (PBT) Minimax experiment, the best parameters were an adversary population of size $3$, and an agent population of size $3$. All adversaries used a convolution kernel size of $3$, an LSTM size of $256$, two fully connected layers of size $32$ each, and a fully connected layer of size $10$ that inputs the timestep and connects to the LSTM. 

For the transfer experiments, we first limited the hyperparameters we searched based on those which produced the highest adversary reward, reasoning that these were experiments in which the optimization objective was achieved most effectively. We tested a limited number of hyperparameter settings on a set of validation environments, consisting of a different maze and labyrinth, mazes with 5 blocks and 40 blocks, and a nine rooms environment. We then tested these parameters on novel test environments to produce the scores in the paper. The best parameters for the PAIRED adversary were found to be $64$ convolution filters, entropy regularization of $0.1$, no RNN, non-negative regret, and having the agents themselves optimize regret. The best parameters for the minimax adversary in this experiment were found to be the same: $64$ convolution filters, entropy regularization of $0.1$, and no RNN. 

\subsection{Hopper Experiments}
The hopper experiments in this paper are in the standard MuJoCo~\citep{todorov2012mujoco} simulator.  The adversary is allowed to apply additional torques to the joints of the agent at a some proportion of the original agent's strength $\alpha$.  The torques that are applied at each time step are chosen by the adversary independent of the state of the environment, and in the PAIRED experiments, both the protagonist and the antagonist are given the same torques.  

The adversaries observation consists of only of the the time step. 
Each policy is a DNN with two hidden layers each with width 32, $tanh$ activation internally and a linear activation on the output layer.  They were trained simultaneously using PPO~\citep{schulman2017proximal} and an schedule in the adversary strength parameter $\alpha$ is scaled from 0.1 to 1.0 over the course of 300 iterations, after which training continues at full strength. We pre-train agents without any adversary for 100 iterations.

Hyperparameter tuning was conducted over the learning rate values [5e-3, 5e-4, 5e-5] and the GAE lambda values $[0.5, 0.9]$.  The minimax adversary worked best with a leaning rate of 5e-4 and a lambda value of $0.5$; PAIRED worked best with a leaning rate of 5e-3 and a lambda value of $0.09$.  Otherwise we use the standard hyperparameters in Ray 0.8.

\end{appendix}

\end{document}

%% file: notation.tex
\usepackage{amsmath}
\usepackage{amsthm}
\usepackage{stmaryrd}

\newtheorem{lemma}{Lemma}
\newtheorem{theorem}[lemma]{Theorem}
\newtheorem{defn}[lemma]{Definition}

        {\medskip}

\usepackage{amsthm}

\newtheorem{innercustomgeneric}{\customgenericname}
\providecommand{\customgenericname}{}
\newcommand{\newcustomtheorem}[2]{%
  \newenvironment{#1}[1]
  {%
   \renewcommand\customgenericname{#2}%
   \renewcommand\theinnercustomgeneric{##1}%
   \innercustomgeneric
  }
  {\endinnercustomgeneric}
}

\newcustomtheorem{customthm}{Theorem}

        {\hspace*{\fill}$\Box$\par\vspace{4mm}}
        {\hspace*{\fill}$\Box$\par}

\mathchardef\dash="2D

\newcommand{\counterfactual}[1]{\ensuremath{%
  \ {\Box}_{#1}\kern-6pt
    \raise1pt\hbox{$\mathord{\longrightarrow}$\ }}}

\DeclareMathOperator*{\argmin}{arg\!\min}

\newcommand{\Dist}[1]{\mathbf{\Delta}(#1)}

\DeclareMathOperator*{\argmax}{arg\!\max}
\newcommand{\EO}{\mathop{\mathbb{E}}}

\newcommand{\Specialize}[2]{{#1}^{#2}}

\newcommand{\PPOMDP}{\mathcal{M}}

\newcommand{\As}{A}
\newcommand{\Os}{O}
\newcommand{\Ss}[1]{\Specialize{S}{#1}}
\newcommand{\Tf}[1]{\Specialize{\mathcal{T}}{#1}}
\newcommand{\Of}[1]{\Specialize{\mathcal{I}}{#1}}
\newcommand{\Rf}[1]{\Specialize{\mathcal{R}}{#1}}
\newcommand{\discount}{\gamma}

\newcommand{\Uf}[1]{\Specialize{U}{#1}}

\newcommand{\Ns}{\Theta}

\newcommand{\apply}[2]{#1_{#2}}